\RequirePackage{fix-cm}
\documentclass[smallcondensed]{svjour3}     
\smartqed  
\usepackage{graphicx}
\usepackage{amssymb}

\usepackage{amsthm}
\usepackage{graphicx}
\usepackage{subcaption}
\graphicspath{{Pics/}}
\usepackage{soul}
\usepackage{xcolor}
\usepackage{bm}
\usepackage{amsmath}
\usepackage{marvosym}

\theoremstyle{plain}

\theoremstyle{definition}

\theoremstyle{remark}

\usepackage[linesnumbered,ruled,vlined]{algorithm2e}
\DeclareMathOperator*{\E}{\mathbb{E}}
\begin{document}

\title{Resolving learning rates adaptively by locating Stochastic Non-Negative Associated Gradient Projection Points using line searches
}


\author{Dominic Kafka         \and
        Daniel N. Wilke 
}


\institute{
	Centre for Asset and Integrity Management (C-AIM),
	Department of Mechanical and Aeronautical Engineering,
	University of Pretoria, Pretoria, South Africa \\
	Tel.: +2712 4202432 \\ 
	Fax: +2712 263 5087 \\
	\email{dominic.kafka@gmail.com}\\
	\email{wilkedn@gmail.com} 
}

\date{Received: date / Accepted: date}

\maketitle

\begin{abstract}
Learning rates in stochastic neural network training are currently determined {\it a priori} to training, using expensive manual or automated iterative tuning. This study proposes gradient-only line searches to resolve the learning rate for neural network training algorithms. Stochastic sub-sampling during training decreases computational cost and allows the optimization algorithms to progress over local minima. However, it also results in discontinuous cost functions. Minimization line searches are not effective in this context, as they use a vanishing derivative (first order optimality condition), which often do not exist in a discontinuous cost function and therefore converge to discontinuities as opposed to minima from the data trends. Instead, we base candidate solutions along a search direction purely on gradient information, in particular by a directional derivative sign change from negative to positive (a Non-negative Associative Gradient Projection Point (NN-GPP)). Only considering a sign change from negative to positive always indicates a minimum, thus NN-GPPs contain second order information. Conversely, a vanishing gradient is purely a first order condition, which may indicate a minimum, maximum or saddle point. This insight allows the learning rate of an algorithm to be reliably resolved as the step size along a search direction, increasing convergence performance and eliminating an otherwise expensive hyperparameter.

\end{abstract}

\section{Introduction}
\label{intro}

The aim of this paper is to compare gradient-only line searches to minimization line searches in the context of the non-convex, discontinuous cost functions encountered in Artificial Neural Network training. With the ability to conduct line searches in this environment, it is possible to automatically resolve step sizes (learning rates) adaptively, as determined by cost function characteristics. To this end, we introduce the various concepts needed for our study.

\subsection{Cost Functions and Gradients}

Machine learning models are often posed with objective functions of the form
\begin{eqnarray}
\mathcal{L}(\mathbf{x}) = \frac{1}{M} \sum_{k=1}^{M} \ell (\mathbf{x};\;\mathbf{t}_k),
\label{eq:loss}
\end{eqnarray}
where $\mathbf{x}\in \mathcal{R}^{d}$ is a vector of parameters, $\{\mathbf{t}_1,\dots,\mathbf{t}_k\}$ is a training set, and $\ell(\mathbf{x};\;\mathbf{t})$ is a loss quantifying the performance of
parameters $\mathbf{x}$ on training sample $\mathbf{t}$. The exact gradient w.r.t.  $\mathbf{x}$ is then computed by back-propagation \cite{Werbos1994}
\begin{eqnarray}
\nabla\mathcal{L}(\mathbf{x}) = \frac{1}{M} \sum_{k=1}^{M} \nabla\ell (\mathbf{x};\;\mathbf{t}_k).
\label{eq:lossgrad}
\end{eqnarray}
For large $M$, the exact gradient can become computationally demanding to compute. Instead, a
(mini-)batch sample, $\mathcal{B} \subset \{1,\dots,M\}$ of size $|\mathcal{B}| \ll M$ is selected from the
training set to compute an approximate loss function
\begin{equation}
L(\mathbf{x}) = \frac{1}{|\mathcal{B}|} \sum_{k\in \mathcal{B}} \ell (\mathbf{x};\;\mathbf{t}_k),
\label{eq:lossgradbatch}
\end{equation}
and approximate stochastic gradient
\begin{equation}
\mathbf{g}(\mathbf{x}) = \frac{1}{|\mathcal{B}|} \sum_{k\in \mathcal{B}} \nabla\ell (\mathbf{x};\;\mathbf{t}_k),
\label{eq:g_lossgradbatch}
\end{equation}
with expectation $\E [ L(\mathbf{x}) ] = \mathcal{L}(\mathbf{x})$  and $\E [ \mathbf{g}(\mathbf{x}) ] = \nabla\mathcal{L}(\mathbf{x})$ \cite{Tong2005}.
Usually, the batch or mini-batch is drawn uniformly and independently from $M$, but it has been noted that mini-batch sampling may lead to the non-convergence of certain algorithms \cite{Balles2018}, which can be rectified using stratified or active mini-batch sampling \cite{Zhang2018}. In addition to cost, mini-batch sampling may alleviate the presence of some local minima, as neural network objective functions have been reported to be multi-modal \cite{Dauphin2014,Goodfellow2015,Choromanska2015,Saxe2013} and therefore aid convergence \cite{Ruder2016}. 

Hence, since $L$ may be non-differentiable due to discontinuities, we weaken the notion of gradients or derivatives to associated derivative and associated gradient \cite{Wilke2013,Snyman2018} as defined in Definitions~\ref{def:ad} and \ref{def:ag}.

\textit{Notation:}
Explicit dependency on variables are occasionally omitted e.g.\ $L$ instead of  $L(\mathbf{x})$. Sequences are denoted by subscripts e.g. $\mathbf{g}_i$, denoting the $i^\textrm{th}$ iterate. Entry-wise products also known as the Hadamard or Schur products \cite{Davis1962} are explicitly denoted using $\odot$. 

\begin{definition}\label{def:ad}
	Let $f:X\subset\mathbb{R}\rightarrow\mathbb{R}$ be a {piecewise} smooth real univariate step-discontinuous function that is everywhere defined. The \emph{associated derivative} $f^{\prime}(x)$ for $f(x)$ at a point $x$ is given by the derivative of $f(x)$ at $x$ when	$f(x)$ is differentiable at $x$. The \emph{associated derivative} $f^{\prime}$ for $f(x)$ non-differentiable at $x$, is given by the left-sided derivative of $f(x)$ when $x$ is associated with the {piecewise} continuous section {of the function to the left of the discontinuity, otherwise it is given by the right-sided derivative.}
\end{definition}

\begin{definition}
	\label{def:ag}
	Let $f:X\subset\mathbb{R}^n\rightarrow\mathbb{R}$ be a {piecewise} continuous function that is everywhere defined. The \emph{associated gradient} $\mathbf{g}(\boldsymbol{x})$ for $f(\boldsymbol{x})$ at $\boldsymbol{x}$ is given by the gradient of $f(\boldsymbol{x})$ at $\boldsymbol{x}$ when $f(\boldsymbol{x})$ is differentiable at $\boldsymbol{x}$. The \emph{associated gradient} $\mathbf{g}(\boldsymbol{x})$ for $f(\boldsymbol{x})$ non-differentiable at $\boldsymbol{x}$ is defined as the vector of partial derivatives with each partial derivative as defined in Definition~\ref{def:ad}. 
\end{definition}
It follows from Definitions~\ref{def:ad} and \ref{def:ag} that the \emph{associated gradient} reduces to the gradient for everywhere differentiable functions. Derivative and gradient is from here on taken to imply respectively associated derivative and associated gradient. Directional derivative denotes a derivative or gradient vector projected along a direction. Gradient vectors are assumed to be column vectors and the vector transpose is indicated by superscript T.

\subsection{Optimization Formulations: Smooth Case Using Full Batch}

In general there are three formulations to define the solution or candidate solutions to an optimization problem using \eqref{eq:loss} and \eqref{eq:lossgrad}, namely:
\begin{enumerate}
	\item Formulation 1: direct minimization of \eqref{eq:loss} \cite{Floudas2009,Arora2011,Snyman2018},
	\item Formulation 2: finding stationary points of \eqref{eq:lossgrad}, known as the optimality criterion \cite{Floudas2009,Arora2011,Snyman2018}, and
	\item Formulation 3: finding non-negative gradient projection points, $\mathbf{x}_{nngpp}$, \cite{Wilke2013,Snyman2018}, as defined by the non-local optimality condition in Definition~\ref{def:generalderivativecriticalpoint}.  
\end{enumerate}

\begin{definition}
	\label{def:generalderivativecriticalpoint}
	Suppose that $f:X\subset \mathbb{R}^n \rightarrow \mathbb{R}$ is a	real-valued function for which the {\emph{associated gradient}} field $\nabla f(\boldsymbol{x})$ is uniquely defined for every $\boldsymbol{x}\in X$. Then, a point $\boldsymbol{x}_{nngpp}\in X$ is a non-negative associated gradient projection point (NN-GPP) if there exists a real number $r_u>0$ for every $\boldsymbol{u}\in\{\boldsymbol{y}\in\mathbb{R}^n \; / \;
	\|\boldsymbol{y}\| = 1\}$ such that	$$\nabla f^\textrm{T}(\boldsymbol{x}_{nngpp} + \lambda\boldsymbol{u})\boldsymbol{u} \geq 0,	\;\forall\;\lambda\in(0,r_u].$$
\end{definition}

The non-negative associated gradient projection point is named after the directional derivative $\nabla f^\textrm{T}(\mathbf{x}_{nngpp} + \lambda\mathbf{u})\mathbf{u}$ that is required to be non-negative for all multi-dimensional directions $\mathbf{u}$, where the gradient is evaluated away from $\mathbf{x}_{nngpp}$ at $\mathbf{x}_{nngpp} + \lambda\mathbf{u}$. 

\begin{figure}[h!]
	\centering
	\begin{subfigure}{.5\textwidth}
		\centering
		\includegraphics[width=0.85\linewidth]{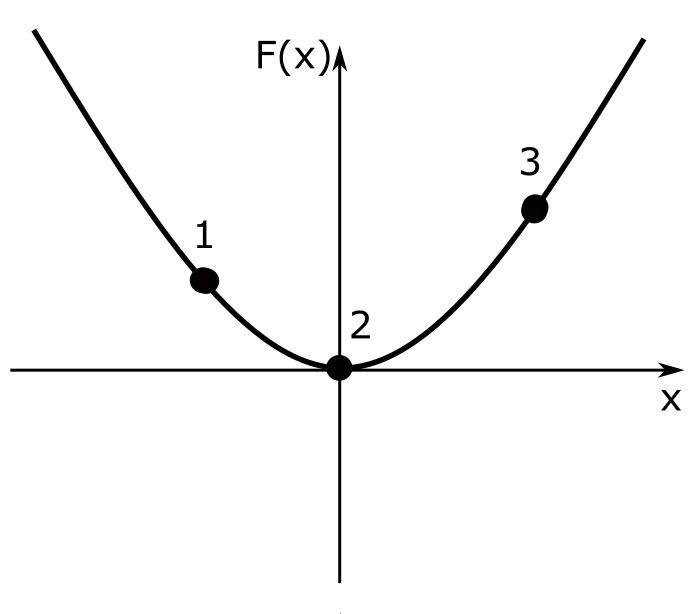}
		\caption{Univariate loss function $F(x)$.}
		\label{fig_ppp_f_diag}
	\end{subfigure}%
	\begin{subfigure}{.5\textwidth}
		\centering 
		\includegraphics[width=0.95\linewidth]{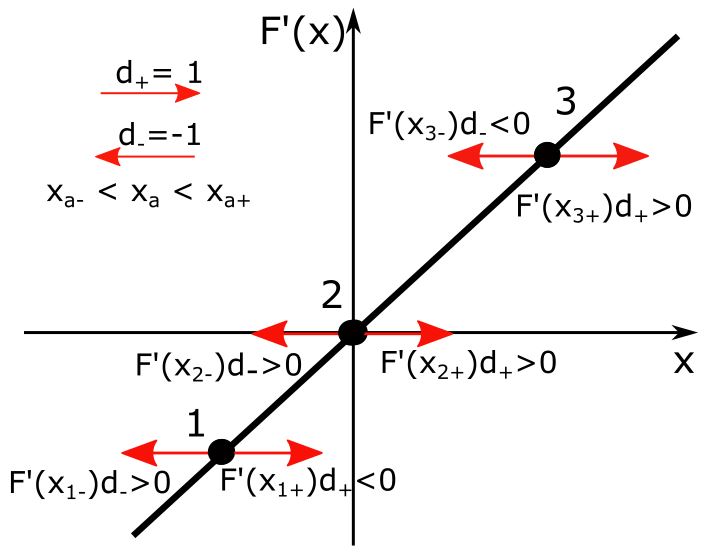}
		\caption{Loss derivative function $F'(x)$.}
		\label{fig_ppp_dd_diag}
	\end{subfigure}%
	\caption{(a) Univariate loss function and (b) derivative function. The non-negative projection point, $x_{nngpp}$, (Formulation 3) for the example quadratic function coincides with $x_{nngpp}= x_2$, at which which the derivative is also zero (Formulation 2) and the function value is a minimum (Formulation 1). Non-negative gradient projection points (NN-GPP) are defined using a non-local optimality formulation that requires all directional derivatives of points in a non-empty ball, $B_\mu(x) = \{x \in \mathcal{R} : |x-x_{nngpp}| <\mu,\;\mu>0,\;x\neq x_{nngpp}\}$ around the NN-GPP to be non-negative. The directional derivative, $F'(x)d$, of a point in the ball $x\in B_\mu$ is defined by the direction vector that connects the non-negative projection point $x_{nngpp}$ to the point in the ball $x$, i.e.\ $d = \frac{x-x_{nngpp}}{\|x-x_{nngpp}\|}$, which is then projected onto the derivative evaluated at the point in the ball $F'(x)$.}
	\label{fig_ppp}
\end{figure}

Formulations 1-3 are well illustrated by the univariate loss function $F(x)$ depicted in Figure~\ref{fig_ppp}(a) with (b) loss derivative function: 
\begin{enumerate}
	\item Formulation 1: $x_2$ is the minimum of $F(x)$.
	\item Formulation 2: The only stationary point (candidate minimum) is $x_2$, since the necessary condition $F'(x_2) = 0$ holds. The candidate solution is an actual minimum if in addition to $F'(x_2) = 0$ the following holds $F''(x_2) > 0$,  which  together forms the sufficiency conditions for $x_2$ to be a local minimum.
	\item Formulation 3: The only non-negative projection point is $x_2$, consider the point $x_1 \in B_\mu(x)$, with one-dimensional direction $d_1 =  \frac{x_1-x_2}{\|x_1-x_2\|} = -1 < 0$ and derivative $F'(x_1) < 0$ that results in the directional derivative $F'(x_1) d_1 > 0$. Similarly,  $x_3 \in B_\mu(x)$, with direction $d_3 =  \frac{x_3-x_2}{\|x_3-x_2\|} = 1 > 0$ and derivative $F'(x_3) > 0$ that results in the directional derivative $F'(x_3) d_3 > 0$. Hence, any point $x \in B_\mu(x)$ results in a non-negative directional derivative. Note that neither $x_1$ or $x_3$ is a non-negative projection point, since there are points to the right of $x_1$ that result in negative directional derivatives ($d > 0$ while $F'(x)<0$) and there are points to the left of $x_3$ that result in negative directional derivatives ($d < 0$ while $F'(x) > 0$). The non-negative projection point $x_2$ defines a local minimum without having to compute second order information \cite{Wilke2013,Snyman2018}, as positive directional derivative implies increase in function along that direction.
\end{enumerate}

Consider Figures~\ref{fig_ppp_example}(a) and (b) for the application of Formulations 1-3 to the actual loss function of the full training set of the Glass1 dataset \cite{Prechelt1994}, for step length $\alpha$ along the steepest descent direction, $-\mathbf{g}(\mathbf{x}_0)$, evaluated at $\mathbf{x}_0$. Figure~\ref{fig_ppp_example}(a)  depicts the actual univariate loss function along the search (descent) direction, i.e.\ $$F(\alpha) = \mathcal{L}(\mathbf{x}_0 - \alpha \mathbf{g}(\mathbf{x}_0)),$$ with directional derivative,$$F'(\alpha) = -\left(\nabla \mathcal{L}(\mathbf{x}_0 - \alpha \nabla \mathcal{L}(\mathbf{x}_0))\right)^\textrm{T} \nabla \mathcal{L}(\mathbf{x}_0),$$ along the search direction depicted in Figure~\ref{fig_ppp_example}(b). The minimum, stationary point and non-negative projection point is at $\alpha = 2.5$. 

Formulations 1-3 have been shown to be equivalent \cite{Wilke2013}, when restricting  \eqref{eq:loss} to the class of convex functions. Formulation 1 is solved by direction minimization of \eqref{eq:loss}, while formulation 2 requires the roots of \eqref{eq:lossgrad} to be computed. Formulation 3 is solved by finding sign changes along directional derivatives of descent directions from negative to positive until no more descent directions remain. Note sign changes from negative to positive along a search direction implies
a minimum without requiring the directional derivative to be zero at the minimum. This distinction is important when considering discontinuous loss functions that occur when conducting mini-batch sampling.

\begin{figure}[h!]
	\centering
	\begin{subfigure}{.5\textwidth}
		\centering
		\includegraphics[width=0.95\linewidth]{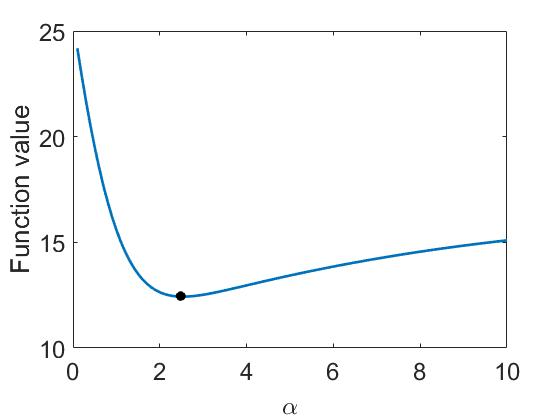}
		\caption{Loss function.}
		\label{fig_full_train_fval}
	\end{subfigure}%
	\begin{subfigure}{.5\textwidth}
		\centering 
		\includegraphics[width=0.95\linewidth]{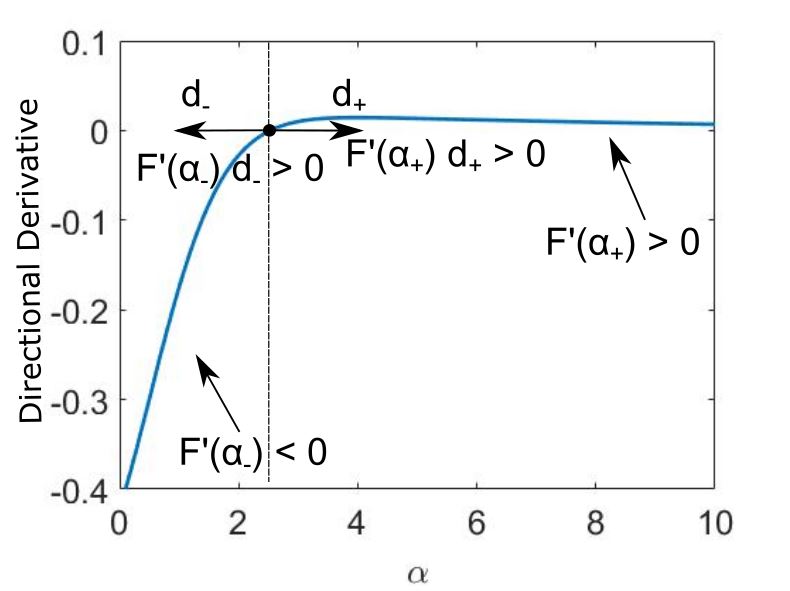}
		\caption{Directional derivative.}
		\label{fig_full_train_grad}
	\end{subfigure}%
	
	\caption{
		Typical loss function along the steepest descent direction in neural net training, with (a) function value and  
		corresponding (b) directional derivative along the descent direction as a function of step size $\alpha$. The problem is constructed using the full training data of the Glass1 dataset \cite{Prechelt1994} (for details see Section 3.5). The local minimum is indicated as a black dot in (a), while the stationary point and non-negative projection point are indicated by a black point in (b).}
	\label{fig_ppp_example}
\end{figure}

\subsection{Optimization Formulations: Discontinuous Case Using Sub-Sampled Mini-Batches}

In the case of randomly sampling a new mini-batch at every approximate loss evaluation and gradient, $L(\boldsymbol{x})$ and $\boldsymbol{g}(\boldsymbol{x})$, discontinuities occur between evaluations. These are stochastic in nature, resulting in spurious local minima being found at discontinuities. Additionally, the probability of a stationary point ($\boldsymbol{g}(\boldsymbol{x}) = 0$) existing at a full-batch minimum, $\boldsymbol{x^*}$, is low, since sub-sampled losses might have gradients close to, but not equal to zero. Although the NN-GPP definition is robust to the presence of discontinuities \cite{Wilke2013,Snyman2018}, the stochasticity introduced due to continually sampling different mini-batches causes the locations of NN-GPPs to no longer be unique. We therefore relax and generalize its definition to the {\it Stochastic} NN-GPP as follows:

\begin{definition}
	\label{def:stochasticderivativecriticalpoint}
	Suppose that $f:X\subseteq \mathbb{R}^p \rightarrow \mathbb{R}$ is a real-valued function for which the \emph{associated gradient} $\nabla f(\boldsymbol{x})$ is uniquely defined for every $\boldsymbol{x}\in X$. Then, a point $\boldsymbol{x}_{snngpp}\in X$ is a stochastic non-negative associated gradient projection point (SNN-GPP) if there exists a real number $r_u>0$ for every $\boldsymbol{u}\in\{\boldsymbol{y}\in\mathbb{R}^p \; / \;
	\|\boldsymbol{y}\| = 1\}$ such that	$${\nabla f}^\textrm{T}(\boldsymbol{x}_{snngpp} + \lambda\boldsymbol{u})\boldsymbol{u} \geq 0,	\;\forall\;\lambda\in(0,r_u],$$ with a
	non-zero probability.
\end{definition}
Therefore, every NN-GPP is also a SNN-GPP, as a NN-GPP satisfies the above definition with probability 1. However, the stochastic nature of the loss results in all SNN-GPPs of a given neighbourhood on the loss surface being contained in a ball with radius smaller than $r_u$. Thus, we defined ball, $B_\epsilon$ as follows:

\begin{definition}
	Consider a mini-batch sub-sampled loss function ${L}(\boldsymbol{x})$, of a continuous, smooth and convex full-batch loss function $\mathcal{L}(\boldsymbol{x})$, such that $\E [ L(\mathbf{x}) ] = \mathcal{L}(\mathbf{x})$  and $\E [ \mathbf{g}(\mathbf{x}) ] = \nabla\mathcal{L}(\mathbf{x})$ \cite{Tong2005}. The minimum of the full-batch loss is $\boldsymbol{x^*}$. Due to the variance of $\mathcal{L}(\mathbf{x})$ at full-batch minimum $\mathbf{x^*}$, there exists a ball, $B_\epsilon(\boldsymbol{x}) = \{\boldsymbol{x} \in \mathbb{R}^p : \|\boldsymbol{x}-\boldsymbol{x}^*\|_2 <\epsilon,\;0<\epsilon<\infty,\;\boldsymbol{x}\neq \boldsymbol{x}^* \}$, that contains all the stochastic non-negative gradient projection points (SNN-GPPs).
	\label{thm:ball}
\end{definition}

To illustrate these concepts, reconsider the Glass1 dataset \cite{Prechelt1994}, with the exception of sampling using mini-batches with stochastic sub-sampling to compute the loss function and gradients. Figures~\ref{fig_minima_comp}(a) and (b) show the function values and directional derivatives along the steepest descent direction for different mini-batch sizes $P \in [10,30,50]$ and compares these to the full dataset (All Data). 

Firstly, it is evident that the loss function and directional derivative along the search direction are both discontinuous as a result of sampling induced discontinuities. The equivalence between Formulations 1-3 having been addressed in the previous section when using the full batch, the differences between Formulations 1-3 are now quantified when using stochastically sub-sampled batches to approximate the loss function and directional derivative along a search direction, as depicted in Figures~\ref{fig_minima_comp}(a) and (b). It is evident that defining the solution as the minimum of \eqref{eq:loss} is both problematic and not representative of the underlying solution as shown by (All Data). Similarly, considering Formulation 2 requires the directional derivative to be zero, which evidently does not exist and also for which the point with smallest directional derivative is not necessarily a good approximation to the All Data solution. Evidently SNN-GPPs are significantly better approximations to the All Data solution.  
SNN-GPPs  are localized in a ball around the All Data solution, 
due to the stochastic nature of the problem. The larger the batch size the smaller the ball radius as indicated by the vertical dotted lines in Figure~\ref{fig_minima_comp}(b) for the corresponding $P$ values averaged over 100 reconstructions.

In addition, we count the average number of local minima and SNN-GPPs  over 100 reconstructions of the line search $F(\alpha)$ over $0 \leq \alpha \leq 10$. For each run along the search direction, 100 steps are taken in increments of 0.1, where a function value and gradient evaluation is performed at each increment. The local minima and SNN-GPPs are then counted within the given range as depicted in Figure~\ref{fig_minima_comp}(c) for the Glass1 dataset. The shaded region gives an indication of the standard deviation and the solid line denotes the mean. When the full dataset is considered, both the number of local minima and SNN-GPP are one and equal to the single local minimum located at $\alpha=2.5$. It is clear, that the SNN-GPPs  indicate only one solution for a much smaller batch size than the minimum of the function. Overall the number of solutions for  SNN-GPPs  are around half the local minima present in the loss function using single sample mini-batches. In addition to the number of solutions, the spatial location of the solutions along the line search is quantified in Figure~\ref{fig_minima_comp}(b) for a batch size of 10. The SNN-GPPs  are localized to the domain $2 \leq \alpha \leq 3$, that is clustered around $\alpha=2.5$, whereas the local minima of the approximated loss function are scattered over almost the entire domain $0.5 \leq \alpha \leq 10$.

\begin{figure}[h!]
	\centering
	\begin{subfigure}{.5\textwidth}
		\centering
		\includegraphics[width=0.95\linewidth]{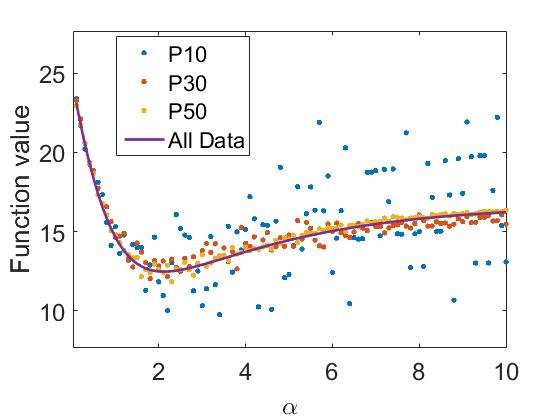}
		\caption{Function values.}
		\label{fig_pcomp_fval}
	\end{subfigure}%
	\begin{subfigure}{.5\textwidth}
		\centering 
		\includegraphics[width=0.95\linewidth]{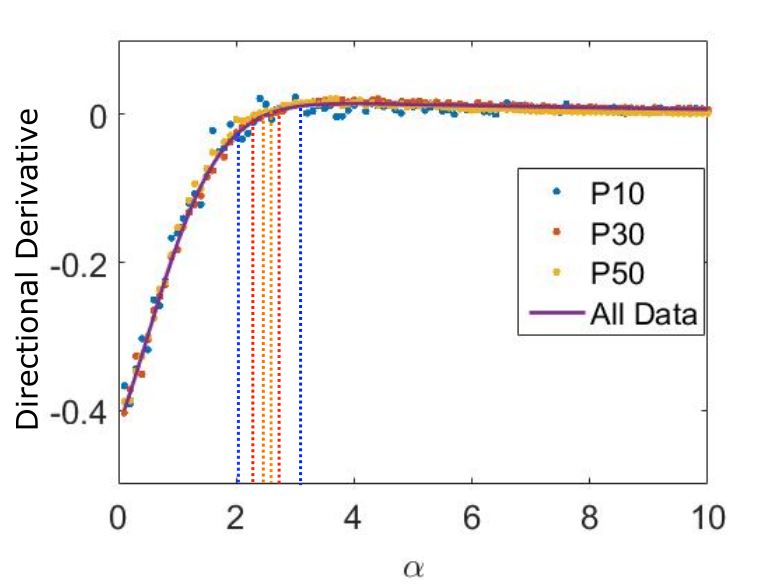}
		\caption{Directional derivatives.}
		\label{fig_pcomp_grad}
	\end{subfigure}%

	\centering
	\begin{subfigure}{.5\textwidth}
		\centering
		\includegraphics[width=0.95\linewidth]{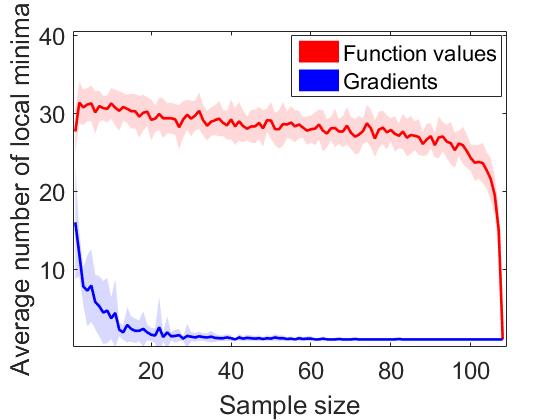}
		\caption{Minima per sample size.}
		\label{fig_glass_l20_K3_stp1}
	\end{subfigure}%
	\begin{subfigure}{.5\textwidth}
		\centering 
		\includegraphics[width=0.95\linewidth]{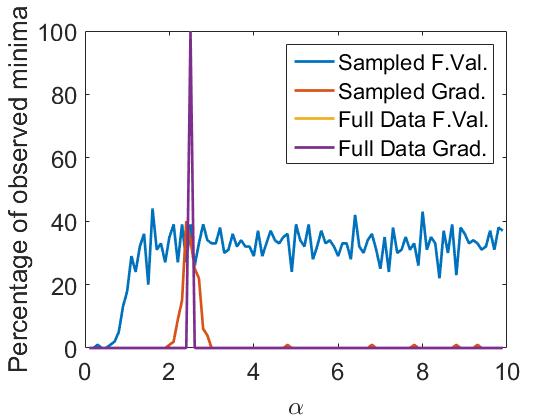}
		\caption{\% local min. along search direction.}
		\label{fig_l1000_inc100_stp1_glass2}
	\end{subfigure}%
	\caption{(a) Loss function and (b) directional derivative along the steepest descent direction with step size $\alpha$ for different mini-batch sizes, namely, 10, 30 and 50 samples per batch stochastically sub-sampled. The size of ball, $B_\epsilon$, along the search direction is also indicated in (b) by vertical lines  for different batch sizes. 
		In addition, the (c) average number of local minima as a function of sample size for $0\leq \alpha \leq 10$, and (d) the spatial location of local minima along the search direction are indicated. Evidently, as the sample size increases, the number as well as location of the minima become consistent. Variance in the directional derivative is far less affected by mini-batch sampling than the loss function, making strategies that sensibly rely on gradient or derivative information better suited for line searches, albeit that the variance has implications for exact and inexact search strategies.}
	\label{fig_minima_comp}
\end{figure}

Evidently, SNN-GPPs  are better generalizers to the All Data solution than the minimum of the approximated loss function. The implication is that line searches that aim at finding SNN-GPPs should perform better than line searches that aim to minimize along a search direction. Away from the domain of SNN-GPPs, sign changes are negligibly scarce in the directional derivative and instead are centered around the All Data solution in a ball, $B_\epsilon$, due to the stochastic nature of the problem. This has implications primarily for exact line searches, where we can only expect to resolve SNN-GPPs to within $B_\epsilon$.

\subsection{Related Work}

Supervised machine learning is often divided into "noisy" stochastic optimization problems \cite{Byrd2012} related to very small batch samples (typically a single data point) and batch averaged approximation regimes that utilize larger batch sizes. As demonstrated in this paper, both problems are discontinuous, with the size of the discontinuities decreasing as the batch size increases. Additional approaches to reduce the discontinuity size include dynamic sample sizes i.e.\ increasing the sample size as the optimizer converges \cite{Byrd2012} that results in a smooth and continuous function in the limit of maximum sample size i.e.\ full batch. As dynamic sub-sampling requires the sample size to grow to the limit of the full batch it is not well suited for mini-batch applications with memory constrained datasets. As demonstrated in this study, the generalization of the optimization problem to find SNN-GPPs  allows for a framework that is well suited for stochastic or active sub-sampling \cite{Zhang2018} over the entire optimization run. Active sub-sampling \cite{Zhang2018} is another approach to reduce the discontinuity size, in which the training data is split into additional datasets based on their error and then effectively sub-sampled stochastically. Active sub-sampling \cite{Zhang2018} can also address the non-convergence issues \cite{Balles2018} of well known machine learning optimizers such as Adam \cite{Kingma2015}. Global optimization strategies such as particle swarm optimization \cite{Engelbrecht2005}, genetic algorithms \cite{Montana1989} and Bayesian combined genetic programming approaches \cite{Marwala2007} are useful in the context of highly non-linear and multi-modal problems while only relying on loss function evaluations. However, these methods are well-known to be computationally demanding. 

Stochastic gradient algorithms were introduced by Robbins-Monro \cite{Robbins1951}, and include developments by Nesterov's dual averaging method \cite{Nesterov2009}. Subgradient methods introduced by Shor \cite{Shor1985,Shor1985a} are closely related to stochastic gradient algorithms. In fact, stochastic gradient descent (SGD) is a classical subgradient method i.e.\ steepest descent with a fixed learning rate \cite{Bertsekas2015}. In turn, subgradient methods are closely related to the newly coined proximal-gradient methods \cite{Bertsekas2015}. All these methods make use of {\it a priori} selected step sizes that are either constant or follow some schedule with known convergence characteristics. These learning rate parameters remain the most sensitive to successful configuration \cite{Bergstra2012} and are currently determined either by the user on a "trial and error" basis, or by computationally expensive automated means \cite{Snoek2012,Bergstra2011,Bergstra2012,Jaderberg2017}. An attempt to incorporate an adaptive learning rate using an inexact line search strategy based on the Wolfe conditions requires dynamic sub-sampling that iteratively increases to the batch size from small to the full batch size \cite{Byrd2012} to finally ensure a smooth and continuous problem that has well-known convergence characteristics. Another approach has been to conduct probabilistic line searches in a Bayesian optimization framework \cite{Mahsereci2017b}. All the approaches discussed so far are referred to as function minimizers i.e.\ they aim to find the minimum function value of the loss function.  

Mini-batch of sub-sampling have seen limited application of line search as they perform poorly or do not converge since the underlying assumptions on which the line searches were developed do not apply for stochastic sub-sampling. Theoretical developments of convergence proofs and estimating theoretical convergence rates include the well-known linear convergence rate of gradient descent methods  for strongly-convex loss functions. Sub-linear convergence rates are achieved when $f$ is only convex \cite{Karimi2016}. A number of alternatives to strong convexity have been presented that include error bounds \cite{Luo1993}, essential strong convexity \cite{Liu2015}, weak strong convexity \cite{Gong2014}, restricted Secant inequality \cite{Zhang2013}, quadratic growth with the Mangasarian-Fromovitz constraint qualification  \cite{Anitescu2000}, Polyak-Lojasiewicz condition \cite{Karimi2016} and associated derivative descent sequences \cite{Wilke2013}. In line with the discontinuous nature of the stochastic loss function, associated derivative descent sequences are not based on assumptions of Lipschitz continuity or convexity but only assumes associated derivative unimodal functions with convergent subsequences.

An alternative to function minimizers are gradient-only optimization strategies that solves for SNN-GPPs  as defined by the gradient-only optimization problem \cite{Wilke2013,Snyman2018}. SNN-GPPs  were specifically proposed to define solutions for discontinuous functions. In this study we consider line searches specifically developed to find SNN-GPPs. We base our convergence proofs on associated derivative descent sequences.

\subsection{Contribution}

The gradient-only optimization problem which solves SNN-GPPs  \cite{Wilke2013,Snyman2018} is a generalized problem, which applies from full batch training (or conventional minimization problem) to mini-batch training (or discontinuous optimization problem). In practice SNN-GPPs have been resolved using gradient-only line search (GOLS) strategies \cite{Wilke2013,Snyman2018}. In this paper we adopt this approach to automatically and adaptively determine the learning rates during supervised mini-batch neural network training, and demonstrate its effectiveness and generality by comparing GOLS to minimization line search strategies over a large number of ANN problems.

\section{Method}

As the emphasis of this study is on resolving the step size and not the implications of search directions, we restrict ourselves
to the mini-batch stochastic gradient descent algorithm. We consider exact and inexact line searches that rely on both loss function values \cite{Floudas2009,Arora2011,Snyman2018} and gradients-only line search strategies \cite{Wilke2013,Snyman2018}. Additional benefits for considering line searches to resolve the step size include the potential to consider higher order algorithms such as conjugate gradient and Quasi-Newton approaches to resolve curvature information regarding the problem \cite{Floudas2009,Arora2011,Snyman2018}.

\subsection{Stochastic Gradient Descent with Line Search Strategies}

Consider stochastic gradient descent (SGD) algorithm as outlined:

\begin{algorithm}
	\DontPrintSemicolon 
	Select starting point, $\mathbf{x}_0$, and set $n = 0$.\;
	Evaluate $L(\mathbf{x}_n)$ and $\mathbf{g}(\mathbf{x}_n)$.\;
	Define the search direction, $\mathbf{d}_n = -\mathbf{g}(\mathbf{x}_n)$.\;
	Assign step length, $\alpha_n$.\;
	Define $\mathbf{x}_{n+1} = \mathbf{x}_n + \alpha_n  \mathbf{d}_n$.\;
	$n = n + 1$,\;
	Continue when stop criterion and limit on number of iterations have not been met else stop.\;
	Repeat steps 2 to 8.\;
	
	\caption{{\sc SGD} Stochastic gradient descent}
	\label{alg_SGD}
\end{algorithm}

Line searches allow for the step size $\alpha_n$ to be resolved along a search direction $\mathbf{d}_n$ from a starting point $\mathbf{x}_n$, which requires resolving $\alpha$ for the univariate function
\begin{equation}
F(\alpha) = L(\mathbf{x}(\alpha)) = L(\mathbf{x}_n + \alpha  \mathbf{d}_n), 
\label{eq_linearised}
\end{equation}
with associated directional derivative along the search direction $\mathbf{d}_n$,
\begin{equation}
F^\prime(\alpha) =  \frac{d F(\alpha)}{d \alpha} = \mathbf{g}(\mathbf{x}(\alpha))^\mathrm{T}\frac{d \mathbf{x}(\alpha)}{d \alpha}  =\mathbf{g}(\mathbf{x}_n + \alpha \cdot \mathbf{d}_n)^\mathrm{T} \mathbf{d}_n.
\label{eq_dlinearised}
\end{equation}
Typically line searches are minimizers i.e.\ find 
\begin{equation}
\arg\min_{\alpha \in \mathcal{R}} f(\alpha),
\end{equation}
which defines functional value based line searches. Alternatively, line searches can be used to find
SNN-GPPs i.e.\ 
\begin{equation}
{\arg \alpha_{nngpp}}  := \left\{\alpha \in \mathcal{R} : F^\prime(\alpha+ \Delta \alpha) \Delta \alpha \geq 0 \right | \forall \; |\Delta\alpha| < r | r > 0 \},
\end{equation}

to form the class of gradient-only line search strategies. For both approaches the desired points can be resolved within a small tolerance to give an exact line search, 
or approximately, which results in inexact line searches.

\subsection{Exact line searches: Function Value based Golden Section (GS) and Bisection Gradient-only Line Search (B-GOLS)}

Function value based line searches were developed for twice differentiable smooth loss functions with the assumption that the bracketing phase only bounds one minimizer. Similarly, B-GOLS was initially developed to bound a single SNN-GPP  for static discontinuous functions. As a result of the stochastic nature of the loss function the usage of an exact line search is a misnomer as the minimum or sign change is assumed to occur in a non-zero $n$-dimensional ball that is an implied dense set. Here, exact line search implies converging to a likely minimizer or 
SNN-GPP within this $n$-dimensional ball. The exact line search strategies are considered for comparative purposes in this study.

Exact line searches first bracket a candidate solution and then refine the interval to find the minimum or SNN-GPP \cite{Arora2011}. Refinement of
a minimum requires significantly more computation than isolating a SNN-GPP  as illustrated in Figures~\ref{fig_diag_exacts}(a) and (b). Four points forming three intervals are required to isolate a local minimum \cite{Floudas2009,Arora2011,Snyman2018}). The optimal reduction is Golden Section reduction that reduces the interval by 38\% per iteration \cite{Floudas2009,Arora2011,Snyman2018}. Isolating a SNN-GPP is equivalent to isolating the directional derivative from negative to positive along the search direction which can be done using bisection i.e.\ three points forming two intervals that reduces the interval by 50\% per iteration \cite{Snyman2018}. Pseudo-code B-GOLS to find SNN-GPP is included in Appendix A. It is important to note that finding only a sign change from negative to positive in the directional derivative along the search direction enhances the robustness of gradient-only approaches against noise, since magnitude variations in directional derivatives are ignored for this exact line search. For both the GS and B-GOLS the step size for the location of minima and SNN-GPP were resolved to a tolerance of $10^{-12}$.

\begin{figure}[h!]
	\centering
	\begin{subfigure}{.5\textwidth}
		\centering
		\includegraphics[width=0.9\linewidth]{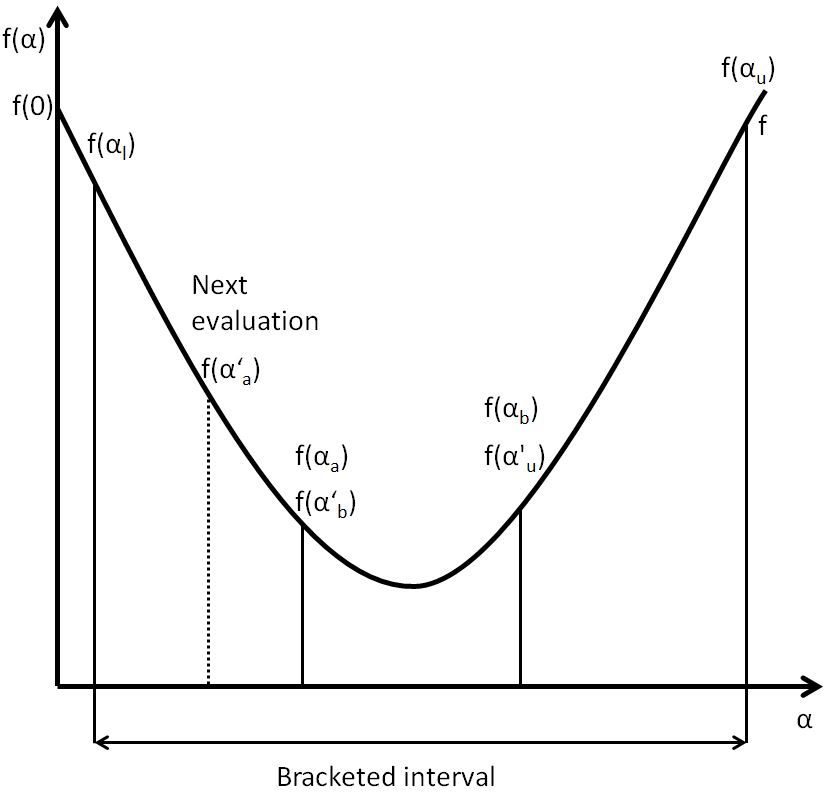}
		\caption{Golden Section}
		\label{fig_diag_golden}
	\end{subfigure}%
	\begin{subfigure}{.5\textwidth}
		\centering 
		\includegraphics[width=0.9\linewidth]{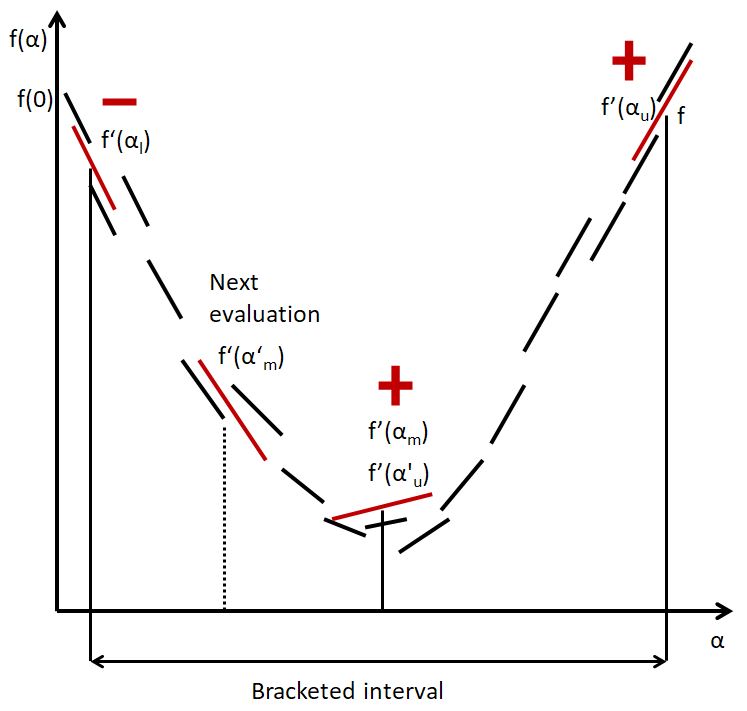}
		\caption{Bisection Gradient-only Line Search}
		\label{fig_diag_grad_exact}
	\end{subfigure}%
	\caption{Comparison between exact line searches that (a) minimize, such as the Golden Section (GS) method, versus (b) identify SNN-GPP by isolating sign changes from negative to positive using the Bisection Gradient-Only Line Search (B-GOLS). The directional derivatives along the search direction $d=1$ at the discrete points are indicated by the red line segments.}
	\label{fig_diag_exacts}
\end{figure}

\subsection{Inexact line search: Function value based Armijo's rule (ARLS) and inexact gradient-only line search (I-GOLS)}

Inexact line searches define ranges for acceptable steps that are not:
\begin{enumerate}
	\item too small by defining a lower bound for the steps, and
	\item not too large by defining an upper bound for the steps.
\end{enumerate}
Ensuring that large enough step sizes are taken is usually enforced by backtracking from large step sizes until an acceptable step size has been
found. Alternatively, step sizes are advanced from small to larger ones until the largest acceptable step size has been found. 
\begin{figure}[h!]
	\centering
	\begin{subfigure}{.5\textwidth}
		\centering
		\includegraphics[width=0.9\linewidth]{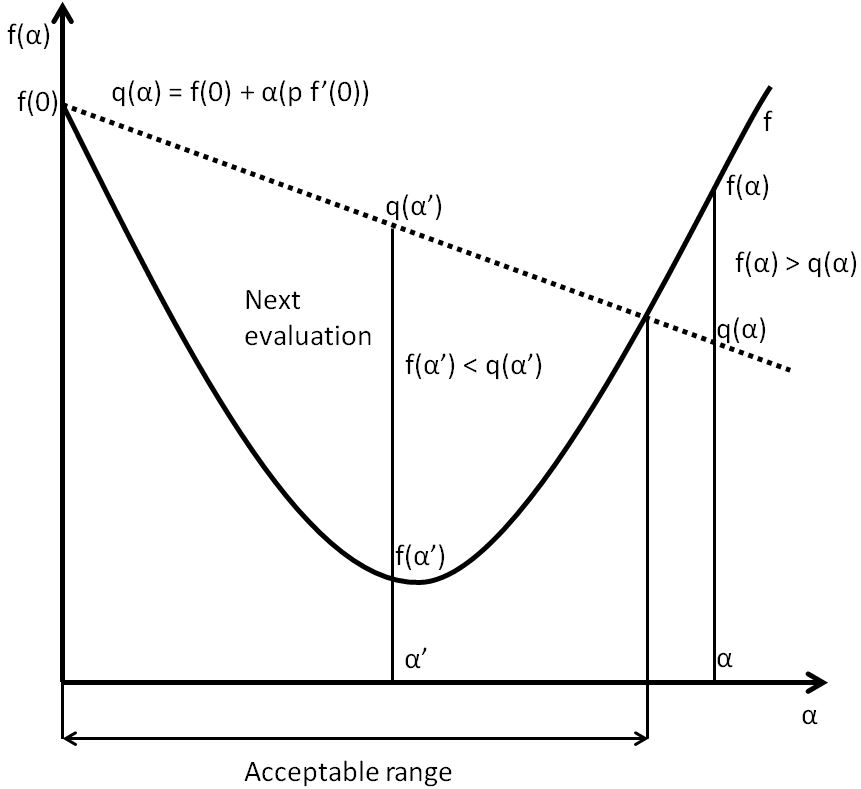}
		\caption{Function value inexact line search.}
		\label{fig_diag_armijo}
	\end{subfigure}%
	\begin{subfigure}{.5\textwidth}
		\centering 
		\includegraphics[width=0.89\linewidth]{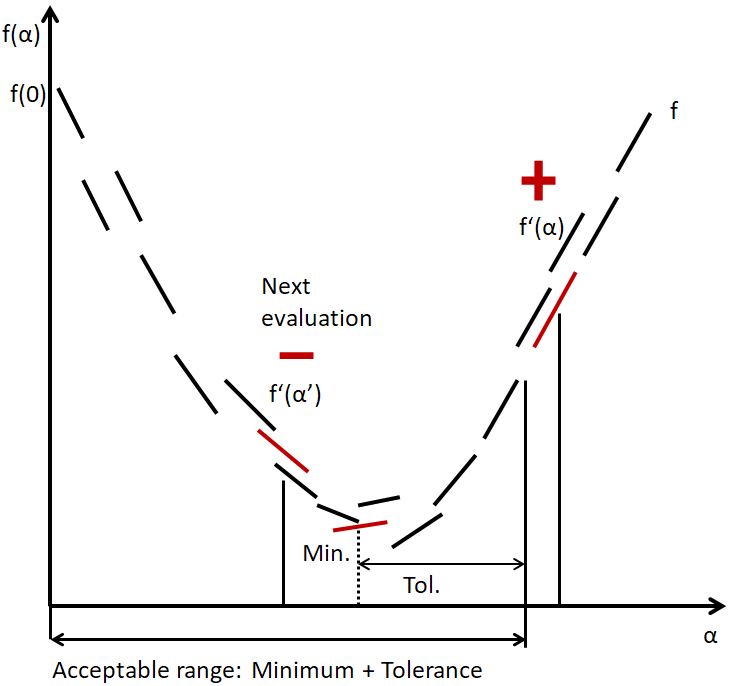}
		\caption{Gradient-only inexact line search.}
		\label{fig_diag_grad_inexact}
	\end{subfigure}%
	\caption{Schematic diagrams for (a)  the function value based inexact line search that is based on Armijo's rule (ARLS) and the (b) inexact gradient-only line search (I-GOLS) with the directional derivative slopes at the points of interest highlighted in red. Armijo's rule was developed for smooth functions, while the gradient-only inexact approach was developed for discontinuous functions.}
	\label{fig_diag_inexacts}
\end{figure}

The chosen function value based inexact line search (ARLS) is based on the practical and robust Armijo's Rule \cite{Arora2011} that defines an upper bound to a domain of acceptable steps
\begin{equation}
F(\alpha) < F(0) + \alpha  (p F'(0)),
\label{eq_qa}
\end{equation}
with $0 \leq p \leq 1$. For $p=0$ any value below the level-set $F(0)$ is allowed, and for $p=1$ the step-size for convex functions is reduced to 0. As a guideline $p=0.2$ is preferred.
In this study, a robust implementation of Armijo's Rule is realized that
ensures the largest feasible step size is taken as the update step. If the initial step is acceptable the step size is increased until the condition is not satisfied and the largest acceptable step taken as the step. Should the initial step fail Armijo's rule, then the step size is reduced using backtracking until the first acceptable step size. This selective employment of advancement or backtracking ensures the largest steps are taken. In this light, increases and decreases of the step size were made with a factor of 2. 

A gradient-only line search that allows for sampling beyond the actual SNN-GPP is simply, any $\alpha$ that satisfies:
\begin{equation}
\frac{d F(\alpha)}{d \alpha} \leq (1-r) |\frac{d F(0)}{d \alpha}|,
\end{equation}
with $0\leq r \leq 1$. Here, $r=0$ implies any update that reduces the magnitude of the directional derivative, whereas $r=1$  implies only derivative descent updates i.e.\ updates for which the directional derivative is negative at the update. Hence, I-GOLS may be sensitive to variance in gradient magnitude. The algorithmic details of this method are given in Appendix B.

As loss functions exist that have unbounded solutions we impose a maximum step size limit $\alpha_{max}$ on both ARLS and I-GOLS, in addition to a minimum step limit $\alpha_{min}$ given by
\begin{eqnarray}
\alpha_{max} &=& \min(\frac{1}{\|\mathbf{g}(\mathbf{x}_n)\|_2}, 10^7),\\
\alpha_{min} &=&  10^{-8}.
\label{eq_max_step}
\end{eqnarray}

The inexact line search strategies require an initial step size. For the very first iteration, this initial guess for $\alpha$ is set to $\alpha_{min}$. In subsequent iterations the initial guess is set to be the resolved step size of the previous iteration, $\alpha_{n-1}$.

\subsection{Theoretical Basis}

Before we present proofs of convergence of associated derivative descent sequences for descendible functions, we  present definitions for associate derivative (weakly) unimodal functions. Associate derivative unimodal functions include discontinuous functions with finite discontinuities but excludes piece-wise linear functions.

\begin{definition}
	\label{def:derivativedescent}
	For a given sequence $\{\boldsymbol{x}^{\{k\}} \in X \subset
	\mathbb{R}^n:k\in\mathbb{P}\}$ suppose $\mathbf{g}(\boldsymbol{x}^{\{k\}}) $ $\neq \boldsymbol{0}$ for some $k$ and
	$\boldsymbol{x}^{\{k\}}\notin B_\epsilon$ with $B_\epsilon$
	defined in Definition~\ref{thm:ball}. Then the
	sequence $\{\boldsymbol{x}^{\{k\}}\}$ is an {associated}
	derivative descent sequence for $f:X\rightarrow\mathbb{R}$, if an
	associated sequence $\{\boldsymbol{u}^{\{k\}} \in
	\mathbb{R}^n:k\in\mathbb{P}\}$ may be generated such that if
	$\boldsymbol{u}^{\{k\}}$ is a descent direction from the set of all
	possible descent directions at $\boldsymbol{x}^{\{k\}}$, i.e.\
	$\mathbf{g}^\textrm{T}
	(\boldsymbol{x}^{\{k\}})\boldsymbol{u}^{\{k\}} <0$ then
	\begin{equation}
	\mathbf{g}^\textrm{T} (\boldsymbol{x}^{\{k+1\}})\boldsymbol{u}^{\{k\}}
	< 0, \text{ for } \boldsymbol{x}^{\{k\}}\neq\boldsymbol{x}^{\{k+1\}}
	\end{equation}
\end{definition}

In this study we limit the functions to the class of descendible associated derivative
multivariate functions, which includes a number of piece-wise smooth step discontinuous functions 
but excludes piece-wise linear continuous functions.
\begin{definition}\label{def:descendable}
	A multivariate function
	$f:X\subset\mathbb{R}^n\rightarrow\mathbb{R}$ is descendible
	if $\mathbf{x}^{\{0\}}\in\mathbb{R}^n$ and
	$\{\mathbf{x}^{\{k\}}\}$ is an {associated} derivative
	descent sequence, as defined in
	Definition~\ref{def:derivativedescent}, for { $f$} with
	initial point $\mathbf{x}^{\{0\}}$, then every subsequence of
	$\{\mathbf{x}^{\{k\}}\}$ converges.  The limit of any convergent
	subsequence of $\{\mathbf{x}^{\{k\}}\}$ is SNN-GPP, as defined in
	Definition~\ref{def:generalderivativecriticalpoint}, of {
		$f$}.
\end{definition}
\begin{note}
	Performing exact gradient-only line searches result in strict associated derivative descent sequences, whilst our proposed inexact line search may or may not depending on the chosen parameter values. Although we consider parameter values in this study for which
	strict associated derivative descent sequences do not hold, the generated subsequences are convergent for the problems considered in this study.
\end{note}

\begin{theorem}
	Let $f:X\subseteq\mathbb{R}^n\rightarrow\mathbb{R}$ be a descendible associated derivative multivariate function as defined by Definition \ref{def:descendable}. Let $\mathbf{x}^{\{0\}}\in\mathcal{R}^n$ and	$\{\mathbf{x}^{\{k\}}\}$ be an {associated} derivative descent sequence, as defined in Definition~\ref{def:derivativedescent}, for {$f$} with initial point $\mathbf{x}^{\{0\}}$, then by definition every subsequence of 		$\{\mathbf{x}^{\{k\}}\}$ converges. The limit of any convergent	subsequence of $\{\mathbf{x}^{\{k\}}\}$ is a ball of {Stochastic non-negative {associated} gradient projection point (SNN-GPP)}, as defined in Definition~\ref{thm:ball}, of {$f$}.
\end{theorem}

\begin{proof}
	It follows our assertion that {$f$} is a descendible associated derivative multivariate function given by Definition~\ref{def:descendable},
	that every derivative descent subsequence $\{\mathbf{x}^{\{k\}}\}$ converges.
	
	Now let $\{{\mathbf{x}}^{\{k\}_m}\}$ be a convergent subsequence of	$\{\mathbf{x}^{\{k\}}\}$ and let $\mathbf{x}^{m*}$ be its limit. Suppose,
	contrary to the second assertion of the theorem, that $\mathbf{x}^{m*}$	is not a {SNN-GPP} as defined in Definition~\ref{def:generalderivativecriticalpoint} of {$f$}. Since we assume that $\mathbf{x}^{m*}$ is not a {SNN-GPP}, and by Definition~\ref{def:derivativedescent}, there exists a $\mathbf{x}^{m*}+\lambda\mathbf{u}$ for $\lambda \neq 0 \in \mathbb{R}$ and $\mathbf{u}\in\{\mathbf{y}\in R^n: \|\mathbf{y}\|_2 = 1 \}$ such that $\mathbf{g}(\mathbf{x}^{m*}+\delta\mathbf{u})\mathbf{u}<0$, which contradicts our assumption that $\mathbf{x}^{m*}$ is the limit of the subsequence $\{{\mathbf{x}}^{\{k\}_m}\}$. Therefore, for $\mathbf{x}^{m*}$ to be the limit of an {associated} derivative descent subsequence $\{\mathbf{x}^{\{k\}_m}\}$,	$\mathbf{x}^{m*} \in B_\epsilon$ as defined in Definition~\ref{thm:ball}, which completes the proof.
\end{proof}

Specific classes of descendible functions can be defined that include
descendible {associated} derivative unimodal multivariate functions.
\begin{definition}
	A multivariate function
	$f:X\subset\mathbb{R}^n\rightarrow\mathbb{R}$ is descendible if $\mathbf{x}^{\{0\}}\in \mathbb{R}^n$ and $\{\mathbf{x}^{\{k\}}\}$ is an {associated} derivative descent sequence, as defined in	Definition~\ref{def:derivativedescent}, for {$f$} with initial point $\mathbf{x}^{\{0\}}$, then every subsequence of $\{\mathbf{x}^{\{k\}}\}$ converges.  The limit of any convergent subsequence of $\{\mathbf{x}^{\{k\}}\}$ is a generalized {SNN-GPP}, as defined in Definition~\ref{def:generalderivativecriticalpoint}, of {$f$} and the set of {SNN-GPP} is compact.
\end{definition}
The class of {associated} derivative unimodal multivariate functions can be relaxed to descendible {associated} derivative
weakly unimodal multivariate functions
\begin{definition}
	A multivariate function
	$f:X\subset\mathbb{R}^n\rightarrow\mathbb{R}$ is
	descendible {associated} derivative according to Definition~\ref{def:descendable}
	and weakly unimodal if the set of SNN-GPP forms a dense set.
\end{definition}  
We consider the latter class of descendible {associated} derivative
weakly unimodal multivariate functions as an appropriate definition for a large class of machine learning 
loss functions. 

\section{Numerical Studies}

\subsection{Diverse Toy Classification Datasets}

In the numerical investigations we consider 22 classification problem datasets that cover from 150 to 70 000 observations per dataset. Our earliest dataset was made available in 1936 and the latest in 2016. The datasets vary in input features from 4 to 784 and classes from 2 to 29. This range in characteristics gives the selected problems a large variety of cost function landscapes for the different optimization algorithms to navigate. The primary aim of our numerical studies is to demonstrate that the performance trends observed are general across vastly different datasets, and representative of ANN problems with non-convex, discontinuous cost function landscapes. Details concerning the datasets are given in Table~\ref{tbl_datasets} and the corresponding neural network architectures implemented for the different datasets are given in Table~\ref{tbl_datasets_NN}. Fully connected neural network layers with one, as well as two hidden layers were considered for all of the datasets but one. For the two hidden layer case, each hidden layer was given the same number of nodes, as shown in Table~\ref{tbl_datasets_NN}. This results in a total of 43 classification problems to be solved. 

\begin{table}[h!]
	\caption{Properties of the datasets considered for the investigation.}
	\centering
	\scalebox{0.7}{
		\begin{tabular}{|c|c|c|c|c|c|c|}
			\hline  
			\multicolumn{6}{|c|}{Dataset properties} \\
			\hline
			Dataset ref. no. & Dataset name & Reference & Observations  & Features & Classes\\ 
			\hline 1 & Cancer1 & \cite{Prechelt1994} & 699  & 9 & 2 \\ 
			\hline 2 & Card1 & \cite{Prechelt1994} & 690  & 51 & 2 \\ 
			\hline 3 & Diabetes1 & \cite{Prechelt1994} & 768  & 8 & 2 \\ 
			\hline 4 & Gene1 & \cite{Prechelt1994} & 3175 & 120 & 3 \\ 
			\hline 5 & Glass1 & \cite{Prechelt1994} & 214 & 9 & 6 \\ 
			\hline 6 & Heartc1 & \cite{Prechelt1994} & 920 & 35 & 2 \\ 
			\hline 7 & Horse1 & \cite{Prechelt1994}& 364 & 58 & 3 \\ 
			\hline 8 & Mushroom1 & \cite{Prechelt1994} & 8124 & 125 & 2 \\ 
			\hline 9 & Soybean1 & \cite{Prechelt1994} & 683 & 35 & 19 \\ 
			\hline 10 & Thyroid1 & \cite{Prechelt1994} & 7200 & 21 & 3 \\ 
			\hline 11 & Abalone & \cite{Nash1994} & 4177 & 8 & 29 \\ 
			\hline 12 & Iris & \cite{Fisher1936} & 150 & 4 & 3 \\ 
			\hline 13 & H.A.R. & \cite{Anguita2012} & 10299 & 561 & 6 \\ 
			\hline 14 & Bankrupted Co. (yr. 1) & \cite{Zieba2016} & 7027 & 64 & 2 \\ 
			\hline 15 & Defaulted Credit Cards & \cite{Yeh2009} & 30000 & 24 & 2 \\ 
			\hline 16 & Forests & \cite{Johnson2012} & 523 & 27 & 4 \\ 
			\hline 17 & FT Clave & \cite{Vurkac2011} & 10800 & 16 & 4 \\ 
			\hline 18 & Sensorless Drive & \cite{Paschke2013} & 58509 & 48 & 11 \\ 
			\hline 19 & Wilt & \cite{Johnson2013} & 4839 & 5 & 2 \\ 
			\hline 20 & Biodegradable compounds & \cite{Mansouri2013} & 1054 & 41 & 2 \\ 
			\hline 21 & Simulation failures & \cite{Lucas2013} & 540 & 20 & 2 \\ 
			\hline 22 & MNIST Handwriting & \cite{Lecun1998} & 70000 & 784 & 10 \\ 
			\hline 
	\end{tabular}}
	\label{tbl_datasets} 
\end{table}

\begin{table}[h!]
	\caption{Properties of the neural network architecture for each dataset.}
	\centering
	\scalebox{0.7}{
		\begin{tabular}{|c|c|c|c|c|}
			\hline  
			\multicolumn{5}{|c|}{ANN properties}\\
			\hline Dataset name & Input nodes  & Hidden layer nodes & Hidden layers & Output nodes\\ 
			\hline Cancer1 & 9 & 8 & 1-2 & 2\\ 
			\hline Card1 & 51 & 5 & 1-2 & 2 \\ 
			\hline Diabetes1 & 8 & 7 & 1-2 & 2 \\ 
			\hline Gene1 & 120 & 9 & 1-2 & 3 \\ 
			\hline Glass1 & 9 & 5 & 1-2 & 6 \\ 
			\hline Heartc1 & 35 & 3 & 1-2 & 2 \\ 
			\hline Horse1 & 58 & 2 & 1-2 & 3 \\ 
			\hline Mushroom1 & 125 & 8 & 1-2 & 2 \\ 
			\hline Soybean1 & 35 & 3 & 1-2 & 19 \\ 
			\hline Thyroid1 & 21 & 8 & 1-2 & 3 \\ 
			\hline Abalone & 8 & 7 & 1-2 & 29 \\ 
			\hline Iris & 4 & 3 & 1-2 & 3 \\ 
			\hline H.A.R. & 561 & 7 & 1-2 & 6\\ 
			\hline Bankrupted Co. (yr. 1) & 64 & 35 & 1-2 & 2\\ 
			\hline Defaulted Credit Cards & 24 & 23 & 1-2 & 2 \\ 
			\hline Forests & 27 & 6 & 1-2 & 4 \\ 
			\hline FT Clave & 16 & 15 & 1-2 & 4 \\ 
			\hline Sensorless Drive & 48 & 47 & 1-2 & 11 \\ 
			\hline Wilt & 5 & 4 & 1-2 & 2 \\ 
			\hline Biodegradable compounds & 41 & 8 & 1-2 & 2 \\ 
			\hline Simulation failures & 20 & 8 & 1-2 & 2 \\ 
			\hline MNIST Handwriting & 784 & 30 & 1 & 10 \\ 
			\hline 
	\end{tabular}}
	\label{tbl_datasets_NN} 
\end{table}

We consider all four optimizers outlined, namely the minimizing line searches (GS and ARLS) and the gradient-only line searches (B-GOLS and I-GOLS). All optimizers are run for 3000 iterations, while the mini-batch size remains constant at 10 data points over all datasets. The mini-batch size is chosen arbitrarily in this case, though arguably all datasets have different characteristics and therefore the optimal mini-batch size to allow for comparable variance between datasets as well as optimizing for computational performance may vary \cite{Radiuk2017,Li}. However, this is outside of the scope of this study.

For all the given datasets, the data was split into training, validation and test datasets such that the ratios between the training dataset to the validation and test datasets are 3:1 respectively. The data is split into three datasets in order to show consistency between the different fractions of the data, should a validation set be selected for a stop criterion and a test set used to evaluate the generality of the model. It is therefore desirable to observe as much consistency as possible between validation and test datasets. The loss function expressed by Equation (\ref{eq_errfunc}) was used to determine the training as well as validation and test dataset errors, using the network in its current configuration at a given iteration. The optimization runs of 3000 iterations are repeated 10 times using exactly the same starting points, determined by generating values between $[-0.1,0.1]$ for each of the neural network weights \cite{Prechelt1994}. This ensures that initially all sigmoid activation functions are in their sensitive domain.

As some methods in this study use function values and other use gradients, a common convention is defined to compare the relative computational cost of the different types information used. Towards this aim, results are given as a function of {\it function evaluations}, where one function value evaluation equates to one function evaluation and one gradient evaluation equates to two function evaluations. A function value evaluation is simply a forward pass to determine the error from the cost function, whereas a gradient evaluation on its own consists of both a forward and a backward pass via backpropagation, thus having twice the computational cost of a function value evaluation. For fair comparison, the number of function evaluations are reported for the different methods and their performances based thereon.

\subsection{Variational Auto Encoder Training}

A separate example is included that demonstrates the performance of I-GOLS in the context of training a variational auto-encoder. The code is written using Pytorch and was sourced from Zou \cite{Zuo}, based on the work of Kingma \cite{Kingma2013}. It was subsequently modified to include and use I-GOLS in SGD. The parameters of I-GOLS were kept the same as discussed previously and the activation functions were again chosen to be sigmoids.

In this investigation we compare I-GOLS to three instances of SGD with fixed step sizes of $\alpha_n = 1\cdot 10^{-6}$, $\alpha_n = 1\cdot 10^{-5}$ and $\alpha_n = 1\cdot 10^{-4}$. These learning rates were manually chosen and selected such that one learning rate is too high, one appropriate and one too low, each separated by an order of magnitude.

\section{Results}

The results are split into different error plots for the training, validation and test datasets respectively. This shows the effectiveness of the different methods during training, and displays the consistency of obtained solutions to the unseen data. It is important to note here, that since the optimizers operate only on the training data, their specific performances are gauged on the training data errors. The validation and test dataset error plots give indications on the generalization of the network configurations found. Generalization is linked to the characteristics of the {\it individual} dataset (noise, data overlap etc.) and therefore should not be over-emphasised in determining an optimizers performance. This trait underlines the importance of considering a multitude of datasets for determining optimizer performance in this study. However, the trends of validation and test error can give insight into the behaviour of different optimizers and therefore should still be considered.

\subsection{Averaged results}

Figure~\ref{fig_HL_aves} shows the error plots as averaged over all the given datasets for the single hidden layer networks. The solid lines indicate the average errors as defined by the cost function for the different line search algorithms in terms of the number of function evaluations. The shadings surrounding the average error plots indicates the standard deviation of the errors.

\begin{figure}[h!]
	\centering 
	\begin{subfigure}[b]{0.35\textwidth}
		\includegraphics[width=\linewidth]{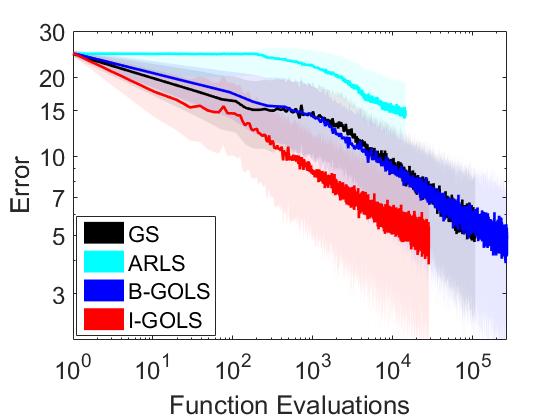}
		\caption{Training data}
	\end{subfigure}%
	\hspace*{\fill}
	\begin{subfigure}[b]{0.35\textwidth}
		\includegraphics[width=\linewidth]{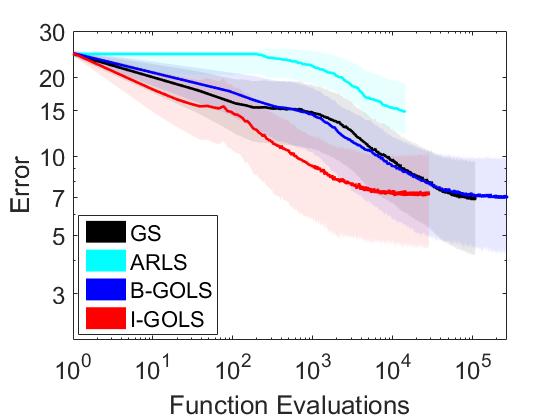}
		\caption{Validation data}
	\end{subfigure}%
	\hspace*{\fill}
	\begin{subfigure}[b]{0.35\textwidth}
		\includegraphics[width=\linewidth]{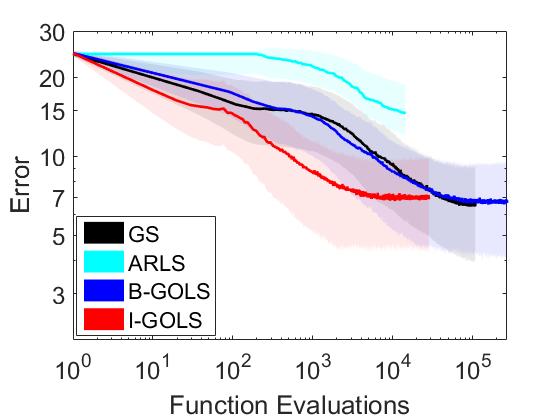}
		\caption{Test data}	    	
	\end{subfigure}%
	
	\centering 
	\begin{subfigure}[b]{0.35\textwidth}
		\includegraphics[width=\linewidth]{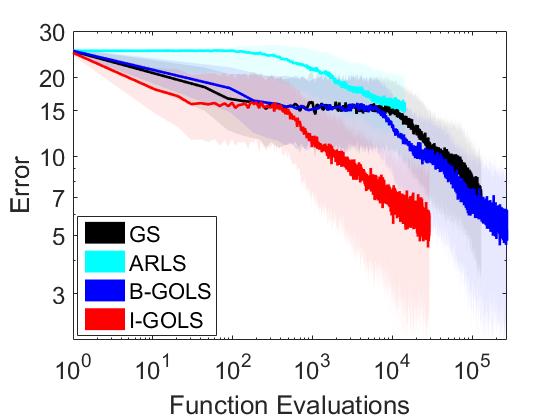}
		\caption{Training data}
	\end{subfigure}%
	\hspace*{\fill}
	\begin{subfigure}[b]{0.35\textwidth}
		\includegraphics[width=\linewidth]{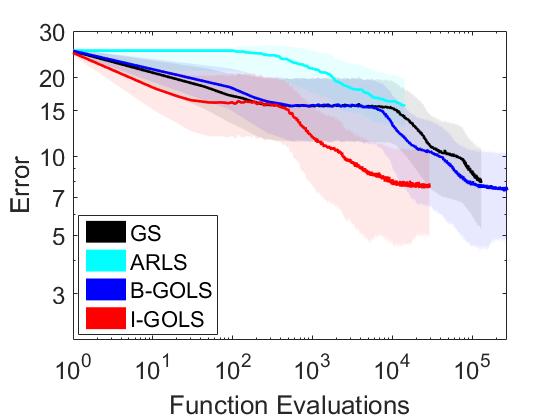}
		\caption{Validation data}
	\end{subfigure}%
	\hspace*{\fill}
	\begin{subfigure}[b]{0.35\textwidth}
		\includegraphics[width=\linewidth]{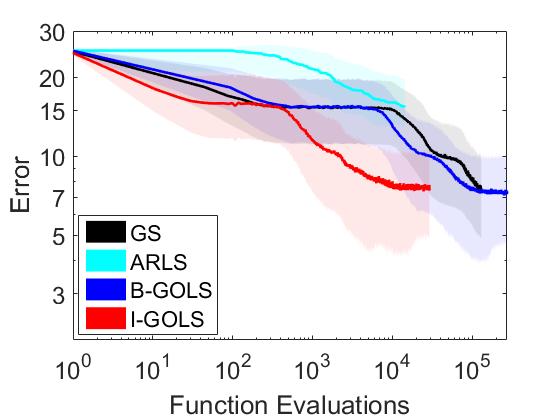}
		\caption{Test data}	    	
	\end{subfigure}%

	\centering 
	\begin{subfigure}[b]{0.35\textwidth}
		\includegraphics[width=\linewidth]{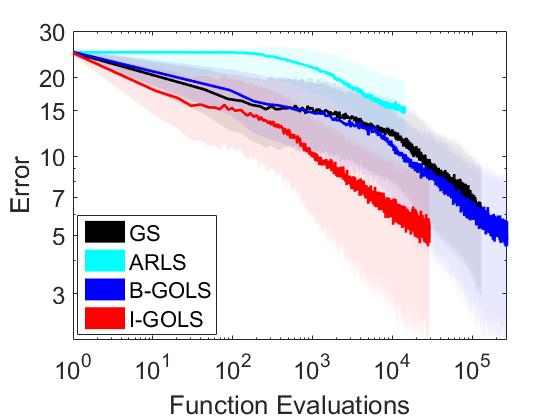}
		\caption{Training data}
	\end{subfigure}%
	\hspace*{\fill}
	\begin{subfigure}[b]{0.35\textwidth}
		\includegraphics[width=\linewidth]{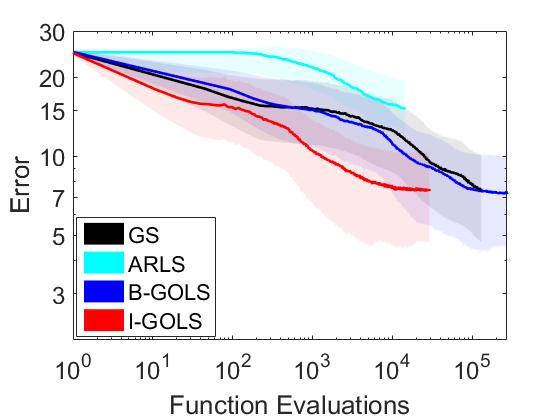}
		\caption{Validation data}
	\end{subfigure}%
	\hspace*{\fill}
	\begin{subfigure}[b]{0.35\textwidth}
		\includegraphics[width=\linewidth]{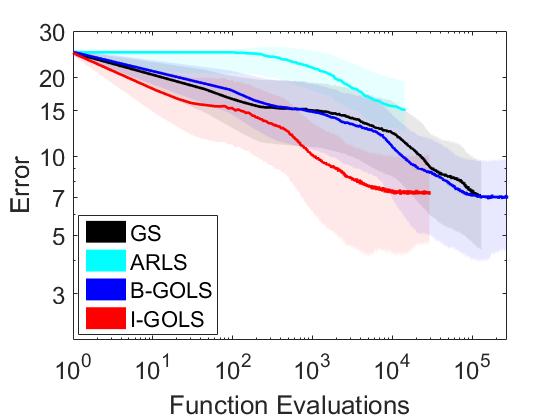}
		\caption{Test data}	    	
	\end{subfigure}%
	\caption{Average training, validation and test dataset error graphs averaged over (a)-(c) single hidden layer networks, (d)-(f) double hidden layer networks and (g)-(i) both single and double hidden layer networks. The graphs are corrected for the computational cost of function value and gradient information. They are given in terms of {\it Function Evaluations}, where a function value evaluation equates to one, and a gradient evaluation equates to two {\it Function Evaluations}.}
	\label{fig_HL_aves}
\end{figure}

The standard deviations are large in Figure~\ref{fig_HL_aves}, as the datasets considered exhibit a wide range of different characteristics, with varying minimum errors and convergence rates. The exact line search algorithms require more function evaluations per iteration than the inexact line searches, as expected. The gradient-only based line searches also require more function evaluations than the equivalent function value based methods. Therefore the question is whether the performance of the gradient-only methods outweighs their added computational cost. 

Consider the case of GS in the single hidden layer average results in Figures~\ref{fig_HL_aves}(a)-(c). Though the training error GS does not progress as far down as that B-GOLS, the validation and test errors are comparatively lower. This may indicate an ability to find better local minima. However, this trend is not general, as can be seen in the case of the double hidden layer networks in Figures~\ref{fig_HL_aves}(d)-(f). In this case B-GOLS is superior in both error and computational cost. This carries over to the average over all problems in Figures~\ref{fig_HL_aves}(g)-(i).

Between the inexact line searches it is clear, that I-GOLS is superior. Function value based ARLS fails to converge to a useful solution within the given number of iterations. The gradient only based I-GOLS produces comparable performance to B-GOLS at an order of magnitude lower number of function evaluations. This is consistent over all three average error plots, demonstrating the ability of I-GOLS to be competitive over both differing datasets as well as network architectures.

\subsection{Examples of individual best performances for different methods}

Within the dataset pool there were isolated incidences where different methods performed better than their averages reflect. In some cases the performance could be deemed better for given methods than the rest. 

This performance was measured by considering training, validation and test errors obtained at the respective computational cost, with a bias towards the errors. If the errors for the different methods were comparable, the best performer is the method with the least computational cost. However, if a clear minimum in error was achieved by a given method in the set number of iterations, that method was considered to be the best performer regardless of its computational cost. Given here are some examples of such cases and the number of their occurrences.

\subsubsection{Golden Section as best performer (3 out of 43 cases)}

\begin{figure}[h!]
	\centering 
	\begin{subfigure}[b]{0.35\textwidth}
		\includegraphics[width=\linewidth]{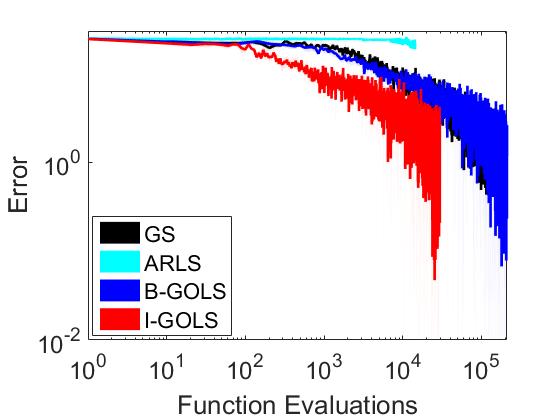}
		\caption{Training data}
	\end{subfigure}%
	\hspace*{\fill}
	\begin{subfigure}[b]{0.35\textwidth}
		\includegraphics[width=\linewidth]{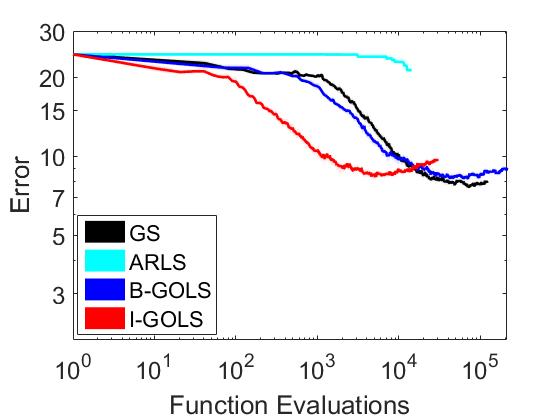}
		\caption{Validation data}
	\end{subfigure}%
	\hspace*{\fill}
	\begin{subfigure}[b]{0.35\textwidth}
		\includegraphics[width=\linewidth]{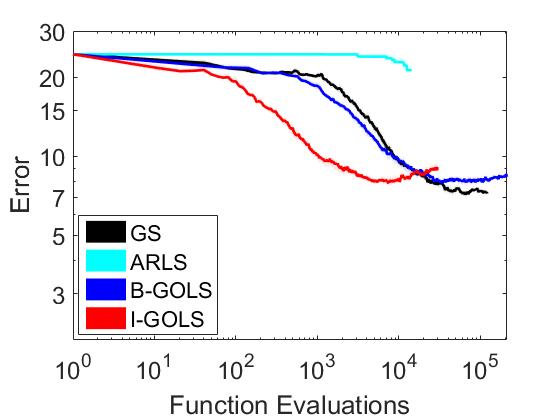}
		\caption{Test data}	    	
	\end{subfigure}%
	\caption{Average (a) training, (b) validation and (c) test dataset errors obtained for different line searches for the Gene1 dataset \cite{Prechelt1994} using the single hidden layer architecture, an example of a dataset where the function value based Golden-Section method was the best performer.}
	\label{fig_best_gs}
\end{figure}

Figure~\ref{fig_best_gs} gives an example where GS was effective in finding local minima. It is clear, that this particular problem is prone to overfitting, as the training error continues to decrease while that of the validation and test datasets exhibits a minimum and sharp increase as training continues. In the training data error plot in Figure~\ref{fig_best_gs}(a), the convergence rate of GS is comparable to that of B-GOLS, albeit that it does not progress as far down in error as B-GOLS for the same number of iterations. This is not so evident in Figure~\ref{fig_best_gs}, but can be deduced from Figures~\ref{fig_best_gs}(b) and (c), where the GS plot ends early. However, the resulting validation and test error minima are on average an error value of 0.8 lower. Similar cases occurred for 3 out of the total of 43 problems, and only on those using single hidden layer networks.

\subsubsection{Bisection Gradient-Only Line Search as best performer (6 out of 43 cases)}

Figure~\ref{fig_best_gob} shows an example of B-GOLS being the most effective in training. This problem exhibits the characteristic where the cost function has a large flat plane to traverse over at an error value around 5. However, after this obstacle has been overcome, convergence is more rapid. It is here where B-GOLS performed the best. Its convergence rate post flat plane is the fastest, obtaining the lowest training as well as validation and test errors for the given number of iterations. The improvement over its nearest rival, I-GOLS is an error value around 0.4. Examples such as these indicate that there are problems which require more accurate resolution of the SNN-GPP in order to progress effectively. However, more iterations would be required in such cases in order to determine the absolute performance of the methods. Training had not yet been completed, thus it is possible that I-GOLS, though slower in convergence after traversing the flat plane, would be computationally cheaper in finding a good SNN-GPP. This would require further investigation. Nevertheless, given the number of fixed iterations, there were 6 datasets on which B-GOLS marginally better than the rest.

\begin{figure}[h!]
	\centering 
	\begin{subfigure}[b]{0.35\textwidth}
		\includegraphics[width=\linewidth]{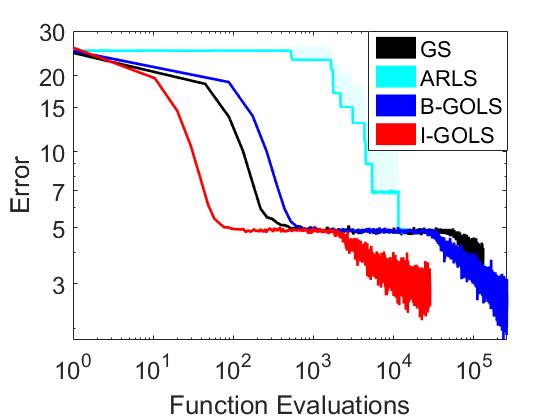}
		\caption{Training data}
	\end{subfigure}%
	\hspace*{\fill}
	\begin{subfigure}[b]{0.35\textwidth}
		\includegraphics[width=\linewidth]{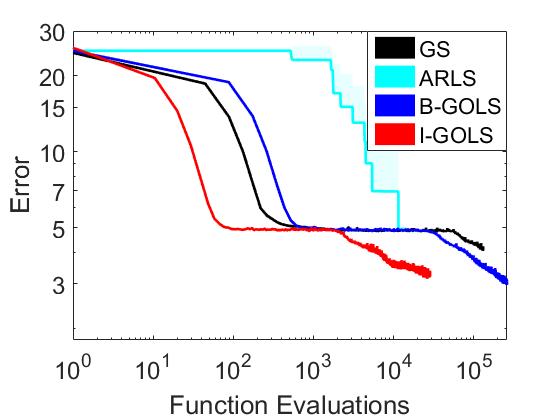}
		\caption{Validation data}
	\end{subfigure}%
	\hspace*{\fill}
	\begin{subfigure}[b]{0.35\textwidth}
		\includegraphics[width=\linewidth]{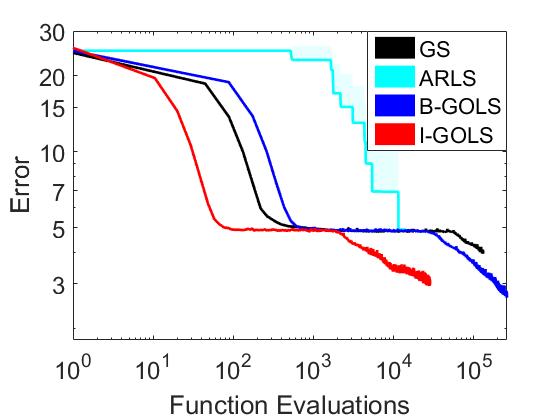}
		\caption{Test data}	    	
	\end{subfigure}%
	\caption{Average (a) training, (b) validation and (c) test dataset errors obtained for different line searches for the Soybean1 dataset \cite{Prechelt1994}, an example of the Bisection Gradient-Only Line Search as the best performer.}
	\label{fig_best_gob}
\end{figure}

\subsubsection{Inexact Gradient-Only Line Search as best performer (34 out of 43 cases)}

\begin{figure}[h!]
	\centering 
	\begin{subfigure}[b]{0.35\textwidth}
		\includegraphics[width=\linewidth]{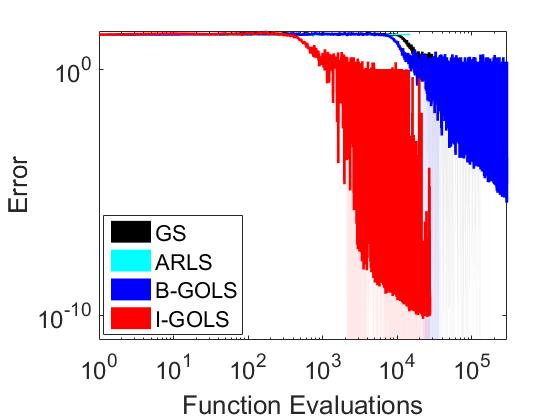} 
		\caption{Training data}
	\end{subfigure}%
	\hspace*{\fill}
	\begin{subfigure}[b]{0.35\textwidth}
		\includegraphics[width=\linewidth]{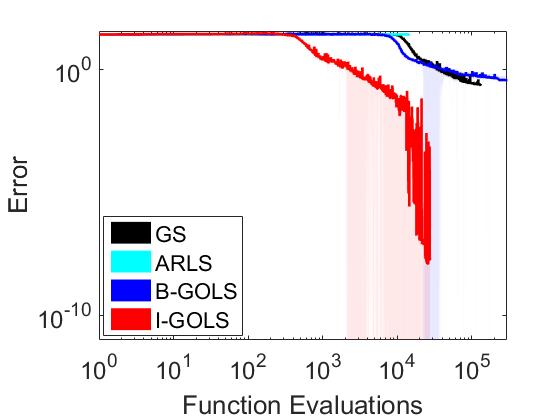}
		\caption{Validation data}
	\end{subfigure}%
	\hspace*{\fill}
	\begin{subfigure}[b]{0.35\textwidth}
		\includegraphics[width=\linewidth]{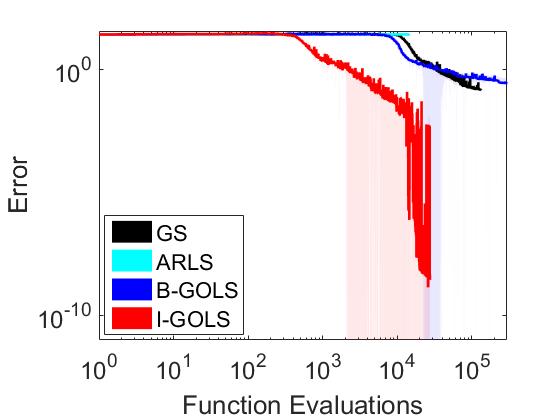}
		\caption{Test data}	    	
	\end{subfigure}%
	\caption{Average (a) training, (b) validation and (c) test dataset errors obtained for different line searches for the Mushroom1 dataset \cite{Prechelt1994}, an example of the inexact gradient-only line search method as the best performer.}
	\label{fig_best_gols}
\end{figure} 

Overall, I-GOLS is the most efficient in progressing training relative to computational cost. As confirmed in the averaged error plots, it is able to reach comparable validation and test error values at an order of magnitude lower function evaluation counts. However, in some cases its performance was considerably better than that of the rest. Figure~\ref{fig_best_gols} shows a case where the progress made is both more efficient as well as superior in convergence. For the given dataset the method progresses quickly towards a good SNN-GPP, causing the error to drop rapidly after overcoming an initial flat plane. In training, the obtained error is 5 orders of magnitude lower than the nearest competitor B-GOLS. This improvement is more pronounced with 9 orders of magnitude for the validation and test dataset errors. This indicates that I-GOLS was superior in training, as well as finding a generalized solution over B-GOLS and GS. This is one of the more extreme examples. However, in the majority of cases the method was capable of returning comparable or better errors than the remaining methods at a order of magnitude lower computational cost, which held for 34 out of 43 of the examined problems.

This shows that the method, though inexact, is capable of reliably traversing the different features of various cost functions relating to a range of datasets and neural network architectures.

\subsection{Comparing step size characteristics of the Glass1 and Cancer1 datasets}

Other matters of interest are the resolved step sizes at every iteration of the different line search methods. Figures~\ref{fig_steps_d1_5}(a)-(f) show the average step sizes, training errors as well as validation errors over 10 runs for the Glass1 and Cancer1 datasets with the double hidden layer architecture. Figure~\ref{fig_steps_d1_5}(a) shows how the step sizes on average vary during an optimization run for the different line search methods on the Glass1 problem. The function value based ARLS produces very small step sizes and therefore does not progress throughout the cost function, as has been previously observed. The GS method is more effective as it produces step sizes with a reasonable magnitude. However, these seem to vary around a mean value. Conversely, I-GOLS and B-GOLS show a distinct variation in step sizes as a function of iterations, first increasing by an order of magnitude into the range of $\alpha_n = 10^1$ before 1000 iterations are reached, then slowly decreasing to a range of $\alpha_n = 10^0$. This behaviour is matched to the shape of the error graphs, where a flat plane is overcome. As this plane is being traversed, the step size is increased. However, once this is overcome, the step size decreases as the optimization methods move along presumably more detailed features in the cost function. This is confirmation that GOLS are able to adjust the step size according to true trends of the discontinuous cost function. It would be difficult to manually pre-empt the requirement of, and determine a learning rate schedule to mimic this behaviour.

From this analysis it is also evident, that I-GOLS is more aggressive than B-GOLS, resulting in larger step sizes. However, this does not negatively impact optimization, as the training and validation errors in Figures~\ref{fig_steps_d1_5}(b) and (c) show I-GOLS to be the overall best performer.

\begin{figure}[h!]
	\centering 
	\begin{subfigure}[b]{0.35\textwidth}
		\includegraphics[width=\linewidth]{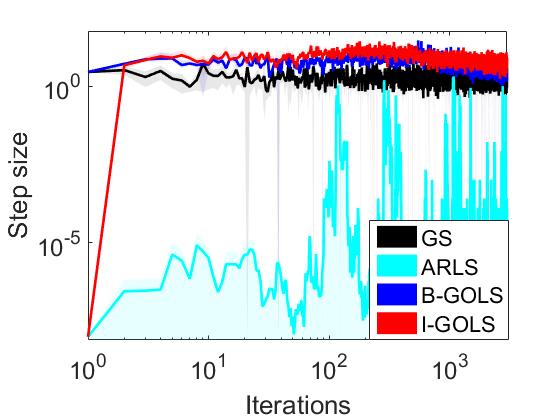}
		\caption{Step sizes}
	\end{subfigure}%
	\hspace*{\fill}
	\begin{subfigure}[b]{0.35\textwidth}
		\includegraphics[width=\linewidth]{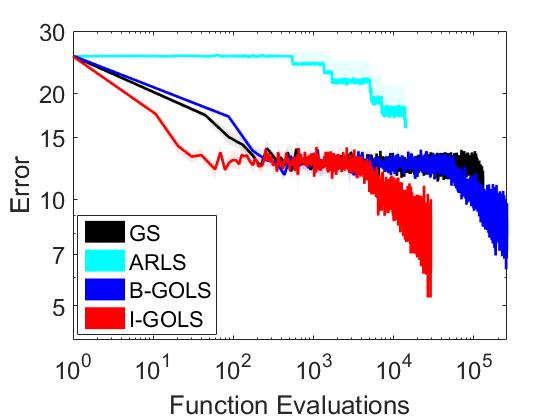}
		\caption{Training error}
	\end{subfigure}%
	\hspace*{\fill}
	\begin{subfigure}[b]{0.35\textwidth}
		\includegraphics[width=\linewidth]{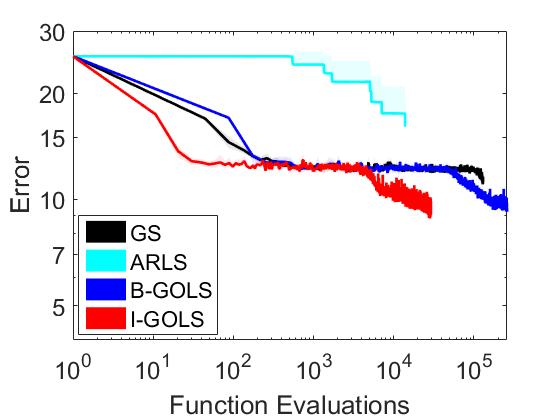}
		\caption{Validation error}	    	
	\end{subfigure}%
	
	\centering 
	\begin{subfigure}[b]{0.35\textwidth}
		\includegraphics[width=\linewidth]{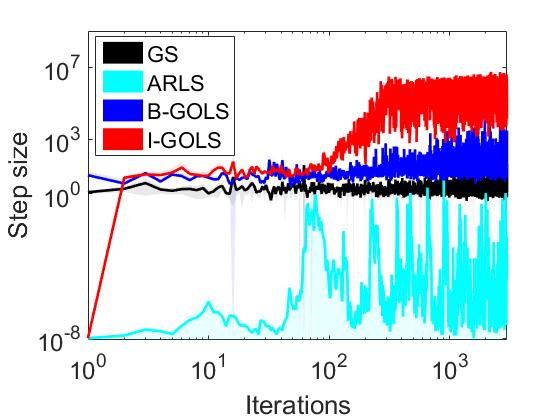}
		\caption{Step sizes}
	\end{subfigure}%
	\hspace*{\fill}
	\begin{subfigure}[b]{0.35\textwidth}
		\includegraphics[width=\linewidth]{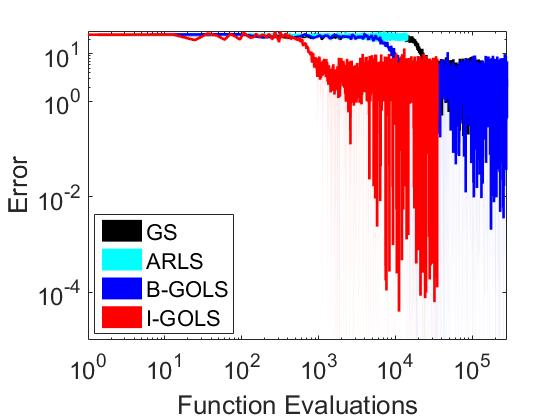}
		\caption{Training error}
	\end{subfigure}%
	\hspace*{\fill}
	\begin{subfigure}[b]{0.35\textwidth}
		\includegraphics[width=\linewidth]{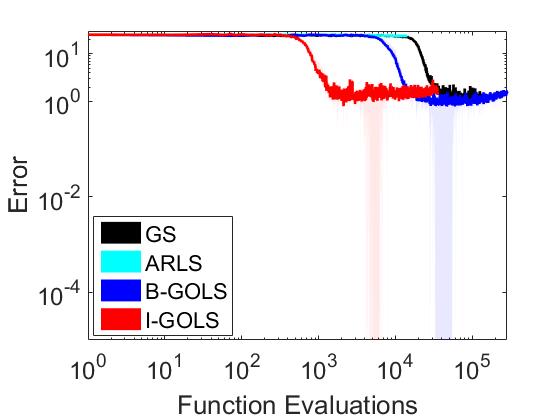}
		\caption{Validation error}	    	
	\end{subfigure}%
	\caption{Step size investigations for (a)-(c) the Glass1 dataset \cite{Prechelt1994} and (d)-(f) the Cancer1 dataset \cite{Prechelt1994} used with the double hidden layer network architecture. Shown are average Step sizes, training errors and validation errors obtained for different line search methods. In this case, the step sizes are given as a function of iterations.}
	\label{fig_steps_d1_5}
\end{figure}

Similarly, in the Cancer1 problem in Figures~\ref{fig_steps_d1_5}(d)-(f), the step sizes of I-GOLS are 4 orders of magnitude larger than those of B-GOLS, ending at step sizes just under $\alpha_n = 10^7$. In this case the gradients decrease towards the end of training, causing I-GOLS to increase the step size to move along the cost function. These large step sizes still relate to good performance on the training loss as depicted in Figure~\ref{fig_steps_d1_5}(e). In terms of validation error, depicted in Figure~\ref{fig_steps_d1_5}(f), a decrease in generalization is observed due to over-fitting of a too flexible neural network for the given data characteristics. However, in the context of training on the data and architecture given, it is evident that I-GOLS remains the best performer.

For the two considered datasets the step sizes exhibited very different non-linearly changing behaviour as a function of the number of iterations. It is practically intractable to determine an equivalent step size rule {\it a priori} for the given datasets, as it would require a substantial effort by the user to establish. This underlines the importance of resolving the step size automatically on an iteration basis via effective line search strategies.

The poor performance of ARLS can be attributed to local minima that result from the discontinuous loss function, in particular, positive jump discontinuities \cite{Wilke2013}. This significantly hampers the growing of the step size from its conservative initial guess. In the case of GS, an exponential bracketing strategy is used, which is less prone to this problem. It is the magnitude of the parameters in the bracketing strategy that aids the progress. The discontinuities cause the bracketing strategy to fall short of encompassing a true minimum. However, if a few probing iterations do not encounter an increasing jump discontinuity, a certain interval may already be defined and subsequently refined to an arbitrary step size therein. Statistically, this makes GS perform similarly to a fixed learning rate, dependent on the magnitude of the parameters chosen for the initial bracketing strategy. However, it is not able to reliably adjust the step size according to features in the cost function.

Conversely, both B-GOLS and I-GOLS is capable of adjusting the step size to the required magnitude within a single optimization iteration, as can be seen for both investigated cases in Figure~\ref{fig_steps_d1_5}(a) and (d). Both gradient based methods converge to similar magnitude step sizes initially. It is in the latter stages of training for the Cancer1 dataset that the two methods diverge in step size.

\subsection{Analysis of iterative performance of line search methods}

It is evident, that there are different computational costs associate with the different investigated line search methods. To quantify this, we counted the average number of function evaluations per iteration and summarize them in Table~\ref{tbl_fvalspinc}. To give context, a single iteration of SGD with a constant learning rate consists of 2 function evaluations, since a gradient computation is required. Function value based line search methods are on average close to half as expensive as the gradient-only based methods. This is due to the added cost of computing the gradient. The inexact methods offer about an order of magnitude computational savings over their exact equivalents. Though comparing function evaluations gives an indication of computational cost of the methods, it does not give a basis for comparison of the iterative efficiency of the respective exact and inexact line search methods. We therefore count the number of "information calls", where a function value and gradient evaluation each constitute a single information call. On this basis the respective exact and inexact line search methods perform very similarly, having similar number of information calls. This shows that algorithmically the methods are comparable. Thus the difference between them is simply a function of the information paradigms used. Even with the added computational cost of the gradient-only line searches, the information gain is substantial enough to offset its cost.

\begin{table}[h!]
	\centering
	\scalebox{0.65}{
		\begin{tabular}{|c|c|c|c|c|}
			\hline   & Func. value exact & Grad.-only exact & Func. value inexact & Grad.-only inexact \\ 
			\hline Ave. no. of function evaluations & 42.8 & 83.3 & 4.75 & 10.4 \\ 
			\hline Ave. no. of information calls & 42.8 & 41.7 & 4.75 & 5.2 \\ 
			\hline
	\end{tabular} }
	\caption{Average number of function evaluations and information calls per iteration for the various line search search optimizers.} 
	\label{tbl_fvalspinc}
\end{table}

Due to the stochastic nature of the loss function it seems more reasonable to consider inexact rather than exact strategies. An exact line search strategy wastes computational resources resolving the accuracy to a bound that is smaller than the variance in the solution due to stochastic subsampling. Gradient-only based I-GOLS is able to bypass discontinuities by observing consistent gradient trends in the loss function, while requiring less gradient evaluations than B-GOLS. It is therefore a plausible method to efficiently resolve the learning rate in the context of discontinuous loss functions, as a result of stochastic data sub-sampling, to sufficient accuracy.

\subsection{Investigation on a Variational Autoencoder example}

\begin{figure}[h!]
	\centering 
	\begin{subfigure}[b]{0.5\textwidth}
		\includegraphics[width=\linewidth]{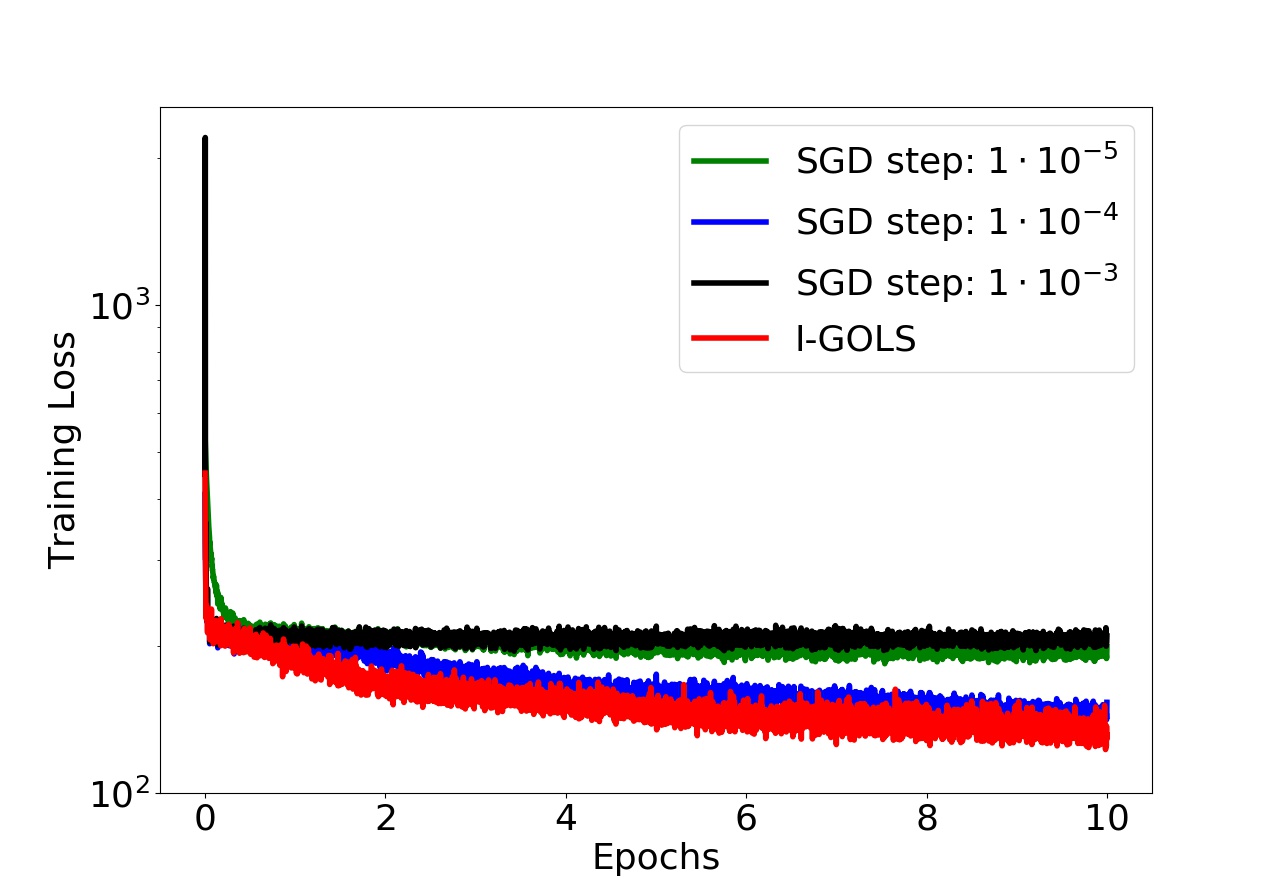}
		\caption{Training Loss}
	\end{subfigure}%
	\hspace*{\fill}
	\begin{subfigure}[b]{0.5\textwidth}
		\includegraphics[width=\linewidth]{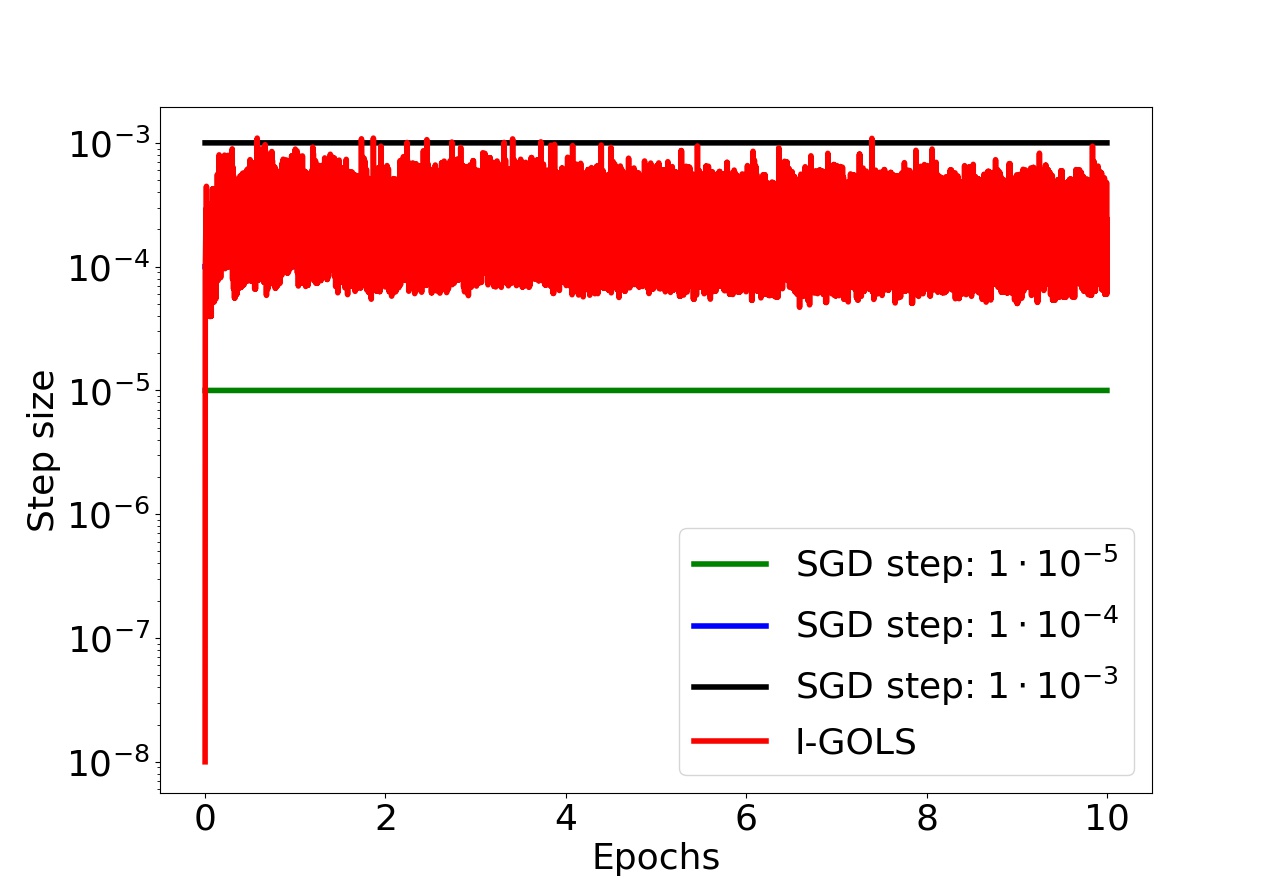}
		\caption{Step Size}
	\end{subfigure}%
	\caption{(a) Training loss and (b) step sizes errors for training a variational autoencoder on MNIST \cite{Lecun1998}. The gradient-only inexact line search is compared to three instances of SGD with different constant learning rates that were determined manually by the user.}
	\label{fig_vae}
\end{figure}

The resulting training loss and corresponding step sizes for the variational autoencoder (VAE) example are given in Figure~\ref{fig_vae}. The training loss in Figure~\ref{fig_vae}(a) shows the performance of the different training methods. The fixed learning rates $\alpha_n = 1 \cdot 10^{-3}$ and $\alpha_n = 1 \cdot 10^{-5}$ are too large and too small respectively. This is indicated by the errors "flat-lining" or decreasing very slowly. However, a fixed step size of $\alpha_n = 1\cdot 10^{-4}$ shows marked improvement in the decrease of the training error. This demonstrates the sensitivity of this problem to the learning rate: It is necessary to determine the value to within an order of magnitude in order to have satisfactory training performance. 
The convergence of I-GOLS is superior to that of the fixed step size of $\alpha_n = 1\cdot 10^{-4}$. 

Recall that these learning rates were manually assigned by the user. In the case of I-GOLS in Figure~\ref{fig_vae}(b), the step size is immediately and automatically adjusted from the initial guess to a range predominantly between $\alpha_n = 1\cdot 10^{-3}$ and $\alpha_n = 10^{-4}$. As training progresses, the step size range slowly decreases with subsequent epochs towards $\alpha_n = 10^{-4}$. This occurs entirely without modification to the parameters in I-GOLS or any other form of user intervention, showing that the SNN-GPP definition has made I-GOLS an effective line search strategy. Arguably one can also use iterative hyperparameter optimization methods \cite{Bergstra2012} to resolve a constant step size for this problem. However, these methods require more than 5 guesses to find an appropriate constant learning rate, making I-GOLS superior in terms of computational cost. In addition, should a problem require an adaptive learning rate, as pointed out, I-GOLS is able to resolve this on an iteration by iteration basis.  

The "manual" optimization process was spread over a range of 3 orders of magnitude and showed the sensitivity therein. Gradient-only based I-GOLS is capable of a step size range over 15 orders of magnitude and was immediately able to resolve a competitive learning rate automatically. This practical example confirms the utility of I-GOLS search in eliminating the learning rate hyperparameter from neural network training while using stochastic sub-sampling.

\section{Conclusion}

Data sub-sampling is readily used to optimize the weights for neural networks, as it speeds up the optimization process for large datasets and allows for memory constrained devices such as GPUs to be more readily employed. The added benefit of adding variance to the function values and gradients, allows optimization strategies to move beyond "weak" local minima or "weak" SNN-GPPs  in the loss function. This comes at the cost of
introducing discontinuities to the loss function that introduces sampling local minima which significantly hampers a function value based minimization line searches. This is evident by the poor performance of both the exact and inexact function value based line searches in our investigations. Both methods are hampered by the multi-modal nature of the objective function used in training. 

In turn, sampling discontinuities do not manifest so readily as SNN-GPPs which are solved for using gradient-only line search (GOLS) strategies. This study demonstrated that in the context of discontinuous loss functions due to stochastic data sub-sampling, the learning rate in steepest gradient descent can be reliably resolved using exact or inexact gradient-only line search strategies over 22 different datasets used in shallow, deep, and variational autoencoder neural network architectures.

Currently, learning rates or learning schedules in ANN training are either selected {\it a priori} and manually adjusted until settings are found that are deemed acceptable by the user, or expensively solved at the global hyperparameter level. This study demonstrated that line searches can be reliably performed to resolve step sizes in discontinuous loss functions as seen in Neural Network training, alleviating the need for {\it a priori} determined step sizes or step size rules. 

The usage of robust line searches may have significant implications for conjugate gradient and Quasi-Newton philosophies to be applied to ANN training.

\begin{acknowledgements}
	\label{Ack}
	This work was supported by the Centre for Asset and Integrity Management (C-AIM), Department of Mechanical and Aeronautical Engineering, University of Pretoria, Pretoria, South Africa.
\end{acknowledgements}


\bibliography{gols} 

\begin{thebibliography}{10}
\providecommand{\url}[1]{{#1}}
\providecommand{\urlprefix}{URL }
\expandafter\ifx\csname urlstyle\endcsname\relax
  \providecommand{\doi}[1]{DOI~\discretionary{}{}{}#1}\else
  \providecommand{\doi}{DOI~\discretionary{}{}{}\begingroup
  \urlstyle{rm}\Url}\fi

\bibitem{Anguita2012}
Anguita, D., Ghio, A., Oneto, L., Parra, X., Reyes-Ortiz, J.L.: {Human activity
  recognition on smartphones using a multiclass hardware-friendly support
  vector machine}.
\newblock Lecture Notes in Computer Science (including subseries Lecture Notes
  in Artificial Intelligence and Lecture Notes in Bioinformatics) \textbf{7657
  LNCS}, 216--223 (2012)

\bibitem{Anitescu2000}
Anitescu, M.: {Degenerate Nonlinear Programming with a Quadratic Growth
  Condition}.
\newblock SIAM Journal on Optimization \textbf{10}(4), 1116--1135 (2000).
\newblock \doi{10.1137/S1052623499359178}.
\newblock \urlprefix\url{http://epubs.siam.org/doi/10.1137/S1052623499359178}

\bibitem{Arora2011}
Arora, J.: {Introduction to Optimum Design, Third Edition}.
\newblock Academic Press Inc (2011).
\newblock
  \urlprefix\url{https://www.sciencedirect.com/science/book/9780128008065
  http://www.amazon.com/Introduction-Optimum-Design-Third-Edition/dp/0123813751}

\bibitem{Balles2018}
Balles, L., Hennig, P.: {Dissecting Adam: The Sign, Magnitude and Variance of
  Stochastic Gradients}.
\newblock arXiv:1705.07774v2 [cs.LG] pp. 1--17 (2018).
\newblock \urlprefix\url{http://arxiv.org/abs/1705.07774}

\bibitem{Bergstra2011}
Bergstra, J., Bardenet, R., Bengio, Y., K{\'{e}}gl, B.: {Algorithms for
  Hyper-Parameter Optimization}.
\newblock In: Advances in Neural Information Processing Systems (NIPS), pp.
  2546--2554 (2011).
\newblock \doi{2012arXiv1206.2944S}

\bibitem{Bergstra2012}
Bergstra, J., Bengio, Y.: {Random Search for Hyper-Parameter Optimization}.
\newblock Journal of Machine Learning Research \textbf{13}(Feb), 281--305
  (2012).
\newblock \doi{10.1162/153244303322533223}.
\newblock \urlprefix\url{http://www.jmlr.org/papers/v13/bergstra12a.html}

\bibitem{Bertsekas2015}
Bertsekas, D.P., {Massachusetts Institute of Technology.}: {Convex optimization
  algorithms}.
\newblock Athena Scientific (2015).
\newblock \urlprefix\url{http://www.athenasc.com/convexalgorithms.html}

\bibitem{Bishop2006}
Bishop, C.M.: {Pattern recognition and machine learning}.
\newblock Springer (2006)

\bibitem{Byrd2012}
Byrd, R.H., Chin, G.M., Nocedal, J., Wu, Y.: {Sample size selection in
  optimization methods for machine learning}.
\newblock Mathematical Programming \textbf{134}(1), 127--155 (2012).
\newblock \doi{10.1007/s10107-012-0572-5}.
\newblock \urlprefix\url{http://link.springer.com/10.1007/s10107-012-0572-5}

\bibitem{Choromanska2015}
Choromanska, A., Henaff, M., Mathieu, M., Arous, G.B., LeCun, Y.: {The Loss
  Surfaces of Multilayer Networks}.
\newblock AISTATS 2015 \textbf{38}, 192--204 (2015).
\newblock \urlprefix\url{http://arxiv.org/abs/1412.0233}

\bibitem{Dauphin2014}
Dauphin, Y., Pascanu, R., Gulcehre, C., Cho, K., Ganguli, S., Bengio, Y.:
  {Identifying and attacking the saddle point problem in high-dimensional
  non-convex optimization}.
\newblock In: ICLR 2014, pp. 1--9 (2014).
\newblock \urlprefix\url{http://arxiv.org/abs/1406.2572}

\bibitem{Davis1962}
Davis, C.: {The norm of the Schur product operation}.
\newblock Numerische Mathematik \textbf{4}(1), 343--344 (1962).
\newblock \doi{10.1007/BF01386329}.
\newblock \urlprefix\url{http://link.springer.com/10.1007/BF01386329}

\bibitem{Engelbrecht2005}
Engelbrecht, A.P.: {Fundamentals of computational swarm intelligence}.
\newblock Wiley (2005).
\newblock \urlprefix\url{http://si.cs.up.ac.za/}

\bibitem{Fisher1936}
Fisher, R.A.: {THE USE OF MULTIPLE MEASUREMENTS IN TAXONOMIC PROBLEMS}.
\newblock Annals of Eugenics \textbf{7}(2), 179--188 (1936).
\newblock \doi{10.1111/j.1469-1809.1936.tb02137.x}.
\newblock
  \urlprefix\url{http://doi.wiley.com/10.1111/j.1469-1809.1936.tb02137.x}

\bibitem{Floudas2009}
Floudas, C.A., Pardalos, P.M. (eds.): {Encyclopedia of Optimization, Second
  Edition}.
\newblock Springer (2009)

\bibitem{Gong2014}
Gong, P., Ye, J.: {Linear Convergence of Variance-Reduced Stochastic Gradient
  without Strong Convexity}.
\newblock ArXiv e-prints  (2014).
\newblock \urlprefix\url{http://arxiv.org/abs/1406.1102}

\bibitem{Goodfellow2015}
Goodfellow, I.J., Vinyals, O., Saxe, A.M.: {Qualitatively Characterizing Neural
  Network Optimization Problems}.
\newblock ICLR pp. 1--11 (2015).
\newblock \urlprefix\url{http://arxiv.org/abs/1412.6544}

\bibitem{Jaderberg2017}
Jaderberg, M., Dalibard, V., Osindero, S., Czarnecki, W.M., Donahue, J.,
  Razavi, A., Vinyals, O., Green, T., Dunning, I., Simonyan, K., Fernando, C.,
  Kavukcuoglu, K.: {Population Based Training of Neural Networks}.
\newblock arxiv pp. 1--13 (2017).
\newblock \urlprefix\url{http://arxiv.org/abs/1711.09846}

\bibitem{Johnson2012}
Johnson, B., Tateishi, R., Xie, Z.: {Using geographically weighted variables
  for image classification}.
\newblock Remote Sensing Letters \textbf{3}(6), 491--499 (2012).
\newblock \doi{10.1080/01431161.2011.629637}.
\newblock
  \urlprefix\url{http://www.tandfonline.com/doi/abs/10.1080/01431161.2011.629637}

\bibitem{Johnson2013}
Johnson, B.A., Tateishi, R., Hoan, N.T.: {A hybrid pansharpening approach and
  multiscale object-based image analysis for mapping diseased pine and oak
  trees}.
\newblock International Journal of Remote Sensing \textbf{34}(20), 6969--6982
  (2013).
\newblock \doi{10.1080/01431161.2013.810825}.
\newblock
  \urlprefix\url{http://www.tandfonline.com/doi/abs/10.1080/01431161.2013.810825}

\bibitem{Karimi2016}
Karimi, H., Nutini, J., Schmidt, M.: {Linear Convergence of Gradient and
  Proximal-Gradient Methods Under the Polyak-{\L}ojasiewicz Condition}.
\newblock In: ECML PKDD: Joint European Conference on Machine Learning and
  Knowledge Discovery in Databases, pp. 795--811. Springer, Cham (2016)

\bibitem{Kingma2015}
Kingma, D.P., Ba, J.: {Adam: A Method for Stochastic Optimization}.
\newblock In: ICLR 2015, pp. 1--15 (2015).
\newblock \doi{http://doi.acm.org.ezproxy.lib.ucf.edu/10.1145/1830483.1830503}.
\newblock \urlprefix\url{http://arxiv.org/abs/1412.6980}

\bibitem{Kingma2013}
Kingma, D.P., Welling, M.: {Auto-Encoding Variational Bayes}.
\newblock arXiv \textbf{Ml}, 1--14 (2013).
\newblock \doi{10.1051/0004-6361/201527329}.
\newblock \urlprefix\url{http://arxiv.org/abs/1312.6114}

\bibitem{Lecun1998}
Lecun, Y., Bottou, L., Bengio, Y., Haffner, P.: {Gradient-based learning
  applied to document recognition}.
\newblock Proceedings of the IEEE \textbf{86}(11), 2278--2324 (1998).
\newblock \doi{10.1109/5.726791}

\bibitem{Li}
Li, M., Zhang, T., Chen, Y., Smola, A.J.: {Efficient Mini-batch Training for
  Stochastic Optimization}.
\newblock In: Proceedings of the 20th ACM SIGKDD international conference on
  Knowledge discovery and data mining (2014).
\newblock \doi{10.1145/2623330.2623612}.
\newblock
  \urlprefix\url{http://www.cs.cmu.edu/{~}muli/file/minibatch{\_}sgd.pdf}

\bibitem{Liu2015}
Liu, J., Wright, S.J., R{\'{e}}, C., Bittorf, V., Sridhar, S.: {An Asynchronous
  Parallel Stochastic Coordinate Descent Algorithm}.
\newblock Journal of Machine Learning Research \textbf{16}, 285--322 (2015).
\newblock \urlprefix\url{http://www.jmlr.org/papers/volume16/liu15a/liu15a.pdf}

\bibitem{Lucas2013}
Lucas, D.D., Klein, R., Tannahill, J., Ivanova, D., Brandon, S., Domyancic, D.,
  Zhang, Y.: {Failure analysis of parameter-induced simulation crashes in
  climate models}.
\newblock Geoscientific Model Development \textbf{6}(4), 1157--1171 (2013).
\newblock \doi{10.5194/gmd-6-1157-2013}.
\newblock \urlprefix\url{http://www.geosci-model-dev.net/6/1157/2013/}

\bibitem{Luo1993}
Luo, Z.Q., Tseng, P.: {Error bounds and convergence analysis of feasible
  descent methods: a general approach}.
\newblock Annals of Operations Research \textbf{46-47}(1), 157--178 (1993).
\newblock \doi{10.1007/BF02096261}.
\newblock \urlprefix\url{http://link.springer.com/10.1007/BF02096261}

\bibitem{Mahsereci2017b}
Mahsereci, M., Hennig, P.: {Probabilistic Line Searches for Stochastic
  Optimization}.
\newblock Journal of Machine Learning Research \textbf{18}(119), 1--59 (2017).
\newblock \doi{10.1016/j.physa.2015.02.029}.
\newblock \urlprefix\url{http://jmlr.org/papers/v18/}

\bibitem{Mansouri2013}
Mansouri, K., Ringsted, T., Ballabio, D., Todeschini, R., Consonni, V.:
  {Quantitative Structure–Activity Relationship Models for Ready
  Biodegradability of Chemicals}.
\newblock Journal of Chemical Information and Modeling \textbf{53}(4), 867--878
  (2013).
\newblock \doi{10.1021/ci4000213}.
\newblock \urlprefix\url{http://pubs.acs.org/doi/10.1021/ci4000213}

\bibitem{Marwala2007}
Marwala, T.: {Bayesian training of neural networks using genetic programming}.
\newblock Pattern Recognition Letters \textbf{28}(12), 1452--1458 (2007).
\newblock \doi{10.1016/J.PATREC.2007.03.004}.
\newblock
  \urlprefix\url{https://www.sciencedirect.com/science/article/pii/S0167865507000967}

\bibitem{Montana1989}
Montana, D.J., Davis, L.: {Training feedforward neural networks using genetic
  algorithms} (1989).
\newblock \urlprefix\url{https://dl.acm.org/citation.cfm?id=1623876}

\bibitem{Nash1994}
Nash, W.J., Sellers, T.L., Talbot, S.R., Cawthorn, A.J., Ford, W.B.: {The
  Population Biology of Abalone ({\_}Haliotis{\_} species) in Tasmania. I.
  Blacklip Abalone ({\_}H. rubra{\_}) from the North Coast and Islands of Bass
  Strait}.
\newblock Tech. rep., Sea Fisheries Division (1994)

\bibitem{Nesterov2009}
Nesterov, Y.: {Primal-dual subgradient methods for convex problems}.
\newblock Math. Program., Ser. B \textbf{120}, 221--259 (2009).
\newblock \doi{10.1007/s10107-007-0149-x}.
\newblock \urlprefix\url{http://ium.mccme.ru/postscript/s12/GS-Nesterov
  Primal-dual.pdf}

\bibitem{Paschke2013}
Paschke, F., Bayer, C., Bator, M., M{\"{o}}nks, U., Dicks, A., Enge-Rosenblatt,
  O., Lohweg, V.: {Sensorlose Zustands{\"{u}}berwachung an Synchronmotoren}
  (2013)

\bibitem{Prechelt1994}
Prechelt, L.: {PROBEN1 - a set of neural network benchmark problems and
  benchmarking rules (Technical Report 21-94)}.
\newblock Tech. rep., Universit\"at Karlsruhe (1994).
\newblock
  \urlprefix\url{http://citeseerx.ist.psu.edu/viewdoc/summary?doi=10.1.1.115.5355}

\bibitem{Radiuk2017}
Radiuk, P.M.: {Impact of Training Set Batch Size on the Performance of
  Convolutional Neural Networks for Diverse Datasets}.
\newblock Information Technology and Management Science \textbf{20}(1), 20--24
  (2017).
\newblock \doi{10.1515/itms-2017-0003}.
\newblock
  \urlprefix\url{http://www.degruyter.com/view/j/itms.2017.20.issue-1/itms-2017-0003/itms-2017-0003.xml}

\bibitem{Robbins1951}
Robbins, H., Monro, S.: {A Stochastic Approximation Method}.
\newblock The Annals of Mathematical Statistics \textbf{22}(3), 400--407
  (1951).
\newblock \doi{10.1214/aoms/1177729586}.
\newblock \urlprefix\url{http://projecteuclid.org/euclid.aoms/1177729586}

\bibitem{Ruder2016}
Ruder, S.: {An overview of gradient descent optimization algorithms}.
\newblock arXiv:1609.04747v2 [cs.LG] pp. 1--14 (2016).
\newblock \doi{10.1111/j.0006-341X.1999.00591.x}.
\newblock \urlprefix\url{http://arxiv.org/abs/1609.04747}

\bibitem{Saxe2013}
Saxe, A.M., McClelland, J.L., Ganguli, S.: {Exact solutions to the nonlinear
  dynamics of learning in deep linear neural networks}.
\newblock CoRR \textbf{abs/1312.6}, 1--22 (2013).
\newblock \urlprefix\url{http://arxiv.org/abs/1312.6120}

\bibitem{Shor1985a}
Shor, N.Z.: In: {Minimization Methods for Non-Differentiable Functions}, chap.
  The Subgra, pp. 22--47. Springer, Berlin Heidelberg (1985).
\newblock
  \urlprefix\url{http://www.springerlink.com/index/10.1007/978-3-642-82118-9{\_}3}

\bibitem{Shor1985}
Shor, N.Z.: {Minimization Methods for Non-Differentiable Functions}.
\newblock Springer, Berlin Heidelberg (1985)

\bibitem{Snoek2012}
Snoek, J., Larochelle, H., Adams, R.: {Practical Bayesian Optimization of
  Machine Learning Algorithms.}
\newblock In: Nips, pp. 1--9 (2012).
\newblock \doi{2012arXiv1206.2944S}.
\newblock
  \urlprefix\url{https://papers.nips.cc/paper/4522-practical-bayesian-optimization-of-machine-learning-algorithms.pdf}

\bibitem{Snyman2018}
Snyman, J.A., Wilke, D.N.: {Practical Mathematical Optimization},
  \emph{Springer Optimization and Its Applications}, vol. 133.
\newblock Springer International Publishing, Cham (2018).
\newblock \doi{10.1007/978-3-319-77586-9}.
\newblock \urlprefix\url{http://link.springer.com/10.1007/978-3-319-77586-9}

\bibitem{Tong2005}
Tong, F., Liu, X.: {Samples Selection for Artificial Neural Network Training in
  Preliminary Structural Design}.
\newblock Tsinghua Science {\&} Technology \textbf{10}(2), 233--239 (2005).
\newblock \doi{10.1016/S1007-0214(05)70060-2}.
\newblock
  \urlprefix\url{https://www.sciencedirect.com/science/article/pii/S1007021405700602}

\bibitem{Vurkac2011}
Vurka{\c{c}}, M.: {Clave-direction analysis: A new arena for educational and
  creative applications of music technology}.
\newblock Journal of Music, Technology and Education \textbf{4}(1), 27--46
  (2011).
\newblock
  \urlprefix\url{http://openurl.ingenta.com/content/xref?genre=article{\&}issn=1752-7066{\&}volume=4{\&}issue=1{\&}spage=27}

\bibitem{Werbos1994}
Werbos, P.J.: {The Roots of Backpropagation: From Ordered Derivatives to Neural
  Networks and Political Forecasting}.
\newblock Wiley-Interscience, New York, NY, USA (1994)

\bibitem{Wilke2013}
Wilke, D.N., Kok, S., Snyman, J.A., Groenwold, A.A.: {Gradient-only approaches
  to avoid spurious local minima in unconstrained optimization}.
\newblock Optimization and Engineering \textbf{14}(2), 275--304 (2013).
\newblock \doi{10.1007/s11081-011-9178-7}.
\newblock \urlprefix\url{http://link.springer.com/10.1007/s11081-011-9178-7}

\bibitem{Yeh2009}
Yeh, I.C., Lien, C.h.: {The comparisons of data mining techniques for the
  predictive accuracy of probability of default of credit card clients}.
\newblock Expert Systems with Applications \textbf{36}(2), 2473--2480 (2009).
\newblock \doi{10.1016/J.ESWA.2007.12.020}.
\newblock
  \urlprefix\url{https://www.sciencedirect.com/science/article/pii/S0957417407006719}

\bibitem{Zhang2018}
Zhang, C., {\"{O}}ztireli, C., Mandt, S., Salvi, G.: {Active Mini-Batch
  Sampling using Repulsive Point Processes}.
\newblock ArXiv e-prints  (2018).
\newblock \urlprefix\url{http://arxiv.org/abs/1804.02772}

\bibitem{Zhang2013}
Zhang, H., Yin, W.: {Gradient methods for convex minimization: better rates
  under weaker conditions}.
\newblock ArXiv e-prints  (2013).
\newblock \urlprefix\url{http://arxiv.org/abs/1303.4645}

\bibitem{Zieba2016}
Zi{\c{e}}ba, M., Tomczak, S.K., Tomczak, J.M.: {Ensemble boosted trees with
  synthetic features generation in application to bankruptcy prediction}.
\newblock Expert Systems with Applications \textbf{58}, 93--101 (2016).
\newblock \doi{10.1016/J.ESWA.2016.04.001}.
\newblock
  \urlprefix\url{https://www.sciencedirect.com/science/article/pii/S0957417416301592}

\bibitem{Zuo}
Zuo, X., Chintala, S.: Basic vae example.
\newblock \url{https://github.com/pytorch/examples/tree/master/vae} (2018).
\newblock Accessed: 2018-06-07

\end{thebibliography}
\bibliographystyle{spmpsci}      


\appendix

\section{Artificial Neural Networks}

The focus of this study is to compare different line search strategies in a stochastic environment and compare their relative performances. The choice of neural network architecture is therefore incidental, as it simply allows for the platform for comparison. Therefore fully connected feed forward single and double hidden-layer neural networks with the sigmoid activation functions were chosen. This architecture is expressed mathematically by Equations (\ref{eq_ANN}) and (\ref{eq_ANN2}) below.

Suppose a given datasets has an input domain, $\mathbf{X}$ with $B$ observations and $D$ dimensions (features). The respective output domain is given by $\mathbf{Y}$ with corresponding observations $B$ and output dimensions $K$, which in our case are representative of the class on one-hot basis. Then for every observation $b$ and every output dimension $k$, a prediction of the output data $\hat{\mathbf{Y}}$ can be constructed from the original data input domain $\mathbf{X}$, given by 

\begin{equation}
\hat{\mathbf{Y}}_{bk} = h_{outer} (\sum_{j=1}^{M_1} \mathbf{W}_{kj}^{(2)} h_{inner} (\sum_{i=1}^{D} \mathbf{W}_{ji}^{(1)} \mathbf{X}_{bi} + \mathbf{W}_{j0}^{(1)} ) + \mathbf{W}_{k0}^{(2)} ),
\label{eq_ANN}
\end{equation}

for a single hidden layer network and 

\begin{equation}
\hat{\mathbf{Y}}_{bk} = h_{outer} (\sum_{l=1}^{M_2} \mathbf{W}_{lk}^{(3)} h_{inner}^{(2)} (\sum_{j=1}^{M_1} \mathbf{W}_{kj}^{(2)} h_{inner}^{(1)} (\sum_{i=1}^{D} \mathbf{W}_{ji}^{(1)} \mathbf{X}_{bi} + \mathbf{W}_{j0}^{(1)} ) + \mathbf{W}_{k0}^{(2)} ) + \mathbf{W}_{l0}^{(3)} ),
\label{eq_ANN2}
\end{equation}

for a double hidden layer network.

The number of nodes in the respective hidden layers is given by $M_{n}$, $n \in [1,2]$. The nodal activation function is denoted by $h$ and $\mathbf{W}$ represents the nodal weights between different layers, where $\mathbf{W}^{(c)}$, $c \in [1,2,3]$ denotes the set of weights connecting sequential layers in the network between the input layer and the output layer in a forward direction. Thus the single hidden layer network has two sets of weights and the double hidden layer network has three respectively \cite{Bishop2006}.

The nodal weights $\mathbf{W}$ are optimized to a configuration which best captures the relationship between the input and output of the data. The cost-function used is the mean least squared error, determined over every $b$ in batch size $P$ and every class $k \in K$ according to the Proben1 dataset guidelines \cite{Prechelt1994} as:

\begin{equation}
E(\mathbf{W}) = \frac{100}{K\cdot P} \sum_{b = 1 }^{P} \sum_{k=1}^{K}(\hat{\mathbf{Y}}_{bk}(\mathbf{W}) - \mathbf{Y}_{bk})^2,
\label{eq_errfunc}
\end{equation}

where $\hat{\mathbf{Y}}(\mathbf{W})$ is the current output of the current network configuration, manipulated as a function of the weights, and $\mathbf{Y}$ is the target output of the corresponding training dataset samples.

\section{Exact line search: Bisection Gradient-only Line Search (B-GOLS)}

The directional derivative values used in this method, denoted $l_{deriv}$, $ m_{deriv}$ and $u_{deriv}$ are each calculated by taking the dot products between the gradient value, $\mathbf{g}(\mathbf{x}_n + \alpha \cdot \mathbf{u}) $ and the search direction, $\mathbf{u}$, at the respective values for $\alpha$ at the different points.

\begin{enumerate}
	\item Define constants: $\delta = 5$, $ r = \frac{\sqrt{5} + 1}{2}$, maximum step size $\alpha_{max}$, minimum step size $\alpha_{min}$, $tol = 10^{-12}$, flag = 1, $k = 0$, $k_{max} = 1000$.
	\item Determine step sizes for: the lower bound, $l = 0$; middle, $m = \delta$ and upper, $ u = m + r \cdot \delta$
	\item flag = 1;
	
	\item evaluate $m_{deriv}$, increment k
	\item evaluate $u_{deriv}$, increment k
	
	\item if $u > \alpha_{max}$
	\begin{enumerate}
		\item $u = \alpha_{max}$
		\item $ I = u - l$
		\item $m=l+\frac{1}{2} I$
		\item evaluate $m_{deriv}$, increment k
		\item evaluate $u_{deriv}$, increment k
	\end{enumerate}	
	
	\item Bracket an interval: while $u_{deriv}$ is negative and flag and $ k < k_{max}$
	\begin{enumerate}
		\item $m = u$
		\item $u = m + r^k \cdot \delta$, where k is the number of function calls
		\item evaluate $ u_{deriv}$, increment k
		\item if  $u > \alpha_{max}$
		\begin{enumerate}
			\item flag = 0
			\item $ \alpha = \alpha_{max} $
		\end{enumerate}
	\end{enumerate}
	
	\item if flag = 1, reduce the interval
	\begin{enumerate}
		\item Define Interval, $I = u - l$
		\item while $I > tol$ and $ u > \alpha_{min}$ and $k < k_{max}$
		\begin{enumerate}
			\item if the sign of $m_{deriv}$ is negative and the sign of $u_{deriv}$ is positive
			\begin{enumerate}
				\item $l = m$				
				\item $I = u - l$
			\end{enumerate}
			\item else if the sign of $m_{deriv}$ is positive
			\begin{enumerate}
				\item $ u = m $
				\item $ u_{deriv} = m_{deriv} $
				\item $ I = u - l $		
			\end{enumerate}
			\item $ m = l + \frac{1}{2}I$
			\item Evaluate the new $m_{deriv}$, increment k
		\end{enumerate}	
		\item finalize the step size: $ \alpha = \frac{u+l}{2}$
	\end{enumerate}
\end{enumerate}

\section{Inexact line search: Inexact Gradient-only Line Search (I-GOLS)}

Parameters used for this method are: $\eta = 2$, $r = 0$, $\alpha_{min} = 10^{-8}$ and $\alpha_{max} = 10^7$.  $F'(\alpha) =  \mathbf{g}(\mathbf{x}_n + \alpha \cdot \mathbf{u})\mathbf{u} $.

\begin{enumerate}
	\item Define constants: $\alpha_{max}$, $\alpha_{min}$, flag = 1, initial guess $\alpha = \alpha_{min}$, $k = 0$, $\eta = 2$, $r=0$
	\item Evaluate $ F'(0) $, increment $k$ 
	\item Evaluate $ F'(\alpha) $, increment $k$
	\item Define $tol_{deriv} = |(1-r)F'(0)|$
	\item if $F'(\alpha) > tol_{deriv}$
	\begin{enumerate}
		\item flag = 1, which will signal a decrease
	\end{enumerate}
	\item if $F'(\alpha) < tol_{deriv}$
	\begin{enumerate}
		\item flag = 2, which will signal an increase
	\end{enumerate}
	\item while flag $>$ 0
	\begin{enumerate}
		\item increment $k$
		\item if flag = 1
		\begin{enumerate}
			\item $\alpha = \frac{\alpha}{\eta} $
			\item Evaluate $ F'(a) $ 
			\item if $F'(\alpha) < tol_{deriv}$
			\begin{enumerate}
				\item flag = 0
			\end{enumerate}
		\end{enumerate}	
		\item if flag = 2
		\begin{enumerate}
			\item $\alpha = \alpha \cdot \eta$
			\item Evaluate $ F'(\alpha) $ 
			\item if $F'(\alpha) > tol_{deriv}$ 
			\begin{enumerate}
				\item $a = \frac{\alpha}{\eta}$
				\item flag = 0
			\end{enumerate}
		\end{enumerate}
		\item if $\alpha < \alpha_{min}$
		\begin{enumerate}
			\item flag = 0
			\item $\alpha = \alpha_{min}$ 
		\end{enumerate}
		\item if $\alpha > \alpha_{max}$
		\begin{enumerate}
			\item flag = 0
			\item $\alpha = \alpha_{max}$ 
		\end{enumerate}
	\end{enumerate}	
\end{enumerate}

\end{document}